\documentclass[a4paper]{article}
\usepackage[margin=1in]{geometry}

\usepackage{microtype}
\usepackage{graphicx}
\usepackage{subcaption}
\usepackage{booktabs} 
\usepackage{here}
\usepackage{authblk}

\usepackage{amsmath,amsthm,amsfonts,amssymb}
\usepackage{mathtools}
\usepackage{xcolor}
\usepackage[sort]{natbib}
\usepackage{algorithm} 
\usepackage{algpseudocode} 
\usepackage{url}
\usepackage{multirow}
\usepackage{hyperref}

\usepackage[T1]{fontenc}
\usepackage[utf8]{inputenc}
\usepackage{bbm}

\usepackage{thmtools}
\usepackage{thm-restate}

\def\*#1{\boldsymbol{#1}}

\theoremstyle{plain}

\newtheorem{theorem}{Theorem}[section]

\newtheorem{lemma}[theorem]{Lemma}

\theoremstyle{definition}
\newtheorem{definition}[theorem]{Definition}
\newtheorem{assumption}[theorem]{Assumption}
\theoremstyle{remark}

\DeclareMathOperator*{\argmin}{arg\,min}
\DeclareMathOperator*{\argmax}{arg\,max}

\newcommand{\lw}[1]{\smash{\lower2.ex\hbox{#1}}}

\newcommand{\rbr}[1]{\left(#1\right)}
\newcommand{\sbr}[1]{\left[#1\right]}
\newcommand{\cbr}[1]{\left\{#1\right\}}

\newcommand{\RR}{\mathbb{R}}
\newcommand{\NN}{\mathbb{N}}

\newcommand{\EE}{\mathbb{E}}

\newcommand{\cD}{{\cal D}}

\newcommand{\cG}{{\cal G}}

\newcommand{\cN}{{\cal N}}

\newcommand{\cP}{{\cal P}}

\newcommand{\cX}{{\cal X}}

\title{Towards Regret Guarantees\\ for One-Step Lookahead Bayesian Optimization}

\date{}

\author[1]{Shion Takeno}

\affil[1]{Nagoya University}

\affil[ ]{\texttt{{takeno.s.mllab.nit@gmail.com}}}

\begin{document}

\maketitle

\begin{abstract}
    This paper studies theoretical guarantees of a one-step lookahead Bayesian optimization (BO) method.
    Although the empirical effectiveness of one-step lookahead BO methods, such as entropy search, has been studied extensively, they often rely on computationally intractable approximations, and their regret guarantees remain underdeveloped.
    Thus, this paper analyzes a one-step lookahead BO method, which we refer to as optimal-point variance reduction (OVR), that requires only posterior sampling and Monte Carlo approximations.
    We obtain a uniform Monte Carlo estimation error bound over an input domain in an acquisition function computation.
    Furthermore, we show that the regularized OVR, with a slight modification to facilitate exploration, achieves a vanishing Bayesian expected simple regret upper bound.
    Finally, we validate the performance of OVR and regularized OVR through numerical experiments.
\end{abstract}
\section{Introduction}
\label{sec:introduction}

Bayesian optimization (BO) \citep{Shahriari2016-Taking} is a framework for the global optimization of expensive-to-evaluate black-box functions.
BO aims to optimize an objective function with fewer function evaluations.
For this purpose, BO sequentially selects an input to evaluate by maximizing acquisition functions (AFs) based on a surrogate model, typically a Gaussian process (GP) \citep{Rasmussen2005-Gaussian}.
BO has been applied to various domains, including hyperparameter tuning in machine learning and experimental design in chemistry \citep{Snoek2012-Practical,korovina2020-chemBO}.

Many AFs for BO evaluate the local utility of candidate inputs using metrics such as improvements or quantities inspired by multi-armed bandit algorithms.
For instance, probability of improvement (PI) and expected improvement (EI) \citep{Kushner1964-new,Mockus1978-Application,Schonlau1998-Global} quantify the potential gain relative to the current best observed value.
Other approaches, such as upper confidence bound (UCB) \citep{Srinivas2010-Gaussian} and Thompson sampling (TS) \citep{Russo2014-learning,takeno2024-posterior}, are inspired by bandit algorithms and are justified by sublinear cumulative regret upper bounds achieved by theoretically balancing exploration and exploitation.

Concurrently, one-step lookahead policies have been extensively studied for their ability to account for global utility across the entire input domain.
Empirical evaluations indicate that one-step lookahead methods, such as entropy search (ES) \citep{Villemonteix2009-aninformational,Henning2012-Entropy,Hernandez2014-Predictive} and knowledge gradient (KG) \citep{Frazier2009-knowledge}, often exhibit superior performance.
However, despite their practical efficacy, the theoretical characterization of one-step lookahead BO remains underdeveloped.
In addition, one-step lookahead policies often suffer from computationally intractable approximations.
Consequently, a gap exists between the practical efficacy of one-step lookahead strategies and their theoretical justification in terms of approximation accuracy and optimization performance.

Towards filling this gap, this paper analyzes optimal-point variance reduction (OVR) \citep{nguyen2017-predictive} and regularized OVR (ROVR), which are one-step lookahead BO methods that aim to minimize the posterior variance or standard deviation (STD) at the optimal input, shown in Figure~\ref{fig:OVR}.
We summarize our contributions below:
\begin{enumerate}
    \item We revisit the easy-to-implement OVR method \citep{nguyen2017-predictive} and propose its regularized variant, ROVR.
    We summarize the preferable properties of OVR and ROVR, such as the small number of approximations compared with ES-based BO methods \citep{Hernandez2014-Predictive,Wang2017-Max,hvarfner2022-joint,tu2022-joint}.
    \item We provide uniform MC estimation error bounds over an input domain in Theorems~\ref{thm:MCError_discrete} and \ref{thm:MCError_continuous}, which indicate $O(\sqrt{\log M / M})$ error convergence even for continuous input domains, where $M$ is the number of MC samples.
    \item We show that ROVR achieves a vanishing expected Bayesian simple regret (BSR) upper bound comparable to most existing results \citep{Russo2014-learning,Takeno2023-randomized,takeno2024-posterior,takeno2025-regret,takeno2025regretEI} in Theorem~\ref{thmBSRBound}.
\end{enumerate}
Finally, we validate the performance of OVR and ROVR via numerical experiments, particularly to assess the effect of regularization.

\begin{figure}[t]
    \centering
    \includegraphics[width=0.6\linewidth]{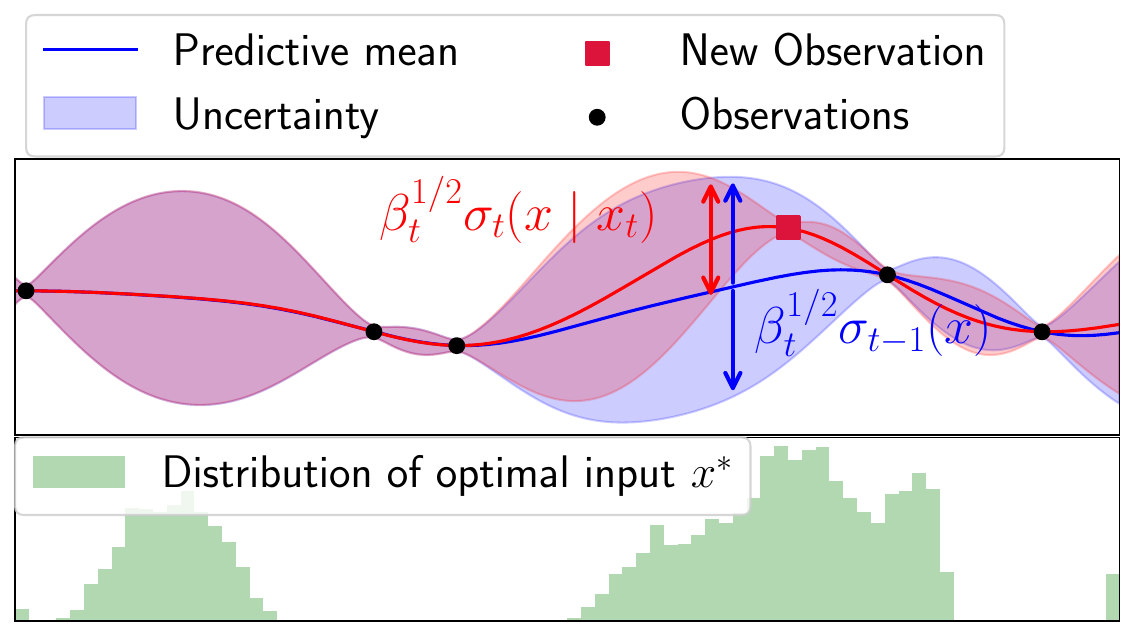}
    \caption{
        Schematic illustration of the OVR.
        The top and bottom figures show the posterior GP model and the posterior distribution of the optimal input $\*x^*$, respectively.
        In the figure, $\beta_t$ is a confidence width parameter and $\sigma_t(\*x \mid \*x_t)$ and $\sigma_{t-1}(\*x)$ represent the posterior STD at $\*x$ after and before observing $\*x_t$, respectively.
        OVR selects the input $\*x_t$ that minimizes the posterior STD at the input that has a higher probability of being the optimal input $\*x^*$.
        That is, OVR aims to minimize the red uncertainty area compared with the blue uncertainty area at the inputs with higher probabilities of being the optimal input.
    }
    \label{fig:OVR}
\end{figure}

\subsection{Related Work}
\label{sec:related}

We analyze the BO algorithm in the Bayesian setting, where the objective function follows a GP \citep{Russo2014-learning,Srinivas2010-Gaussian,iwazaki2025improved}, although the frequentist setting has also been studied extensively \citep{Srinivas2010-Gaussian,Chowdhury2017-on,iwazaki2025-improvedGPbandit}.
Furthermore, although several studies analyze the noiseless setting \citep{iwazaki2025gaussian,freitas2012exponential}, we focus on the noisy setting, where the observation is contaminated by noise.
%

Maximum variance reduction (MVR), also called uncertainty sampling, achieves a convergent simple regret bound \citep{vakili2021-optimal}.
Although the concept is similar to OVR, MVR greedily selects the input with the maximum posterior variance.
Thus, MVR is not a one-step lookahead BO method and cannot incorporate the global utility over the input space.

KG \citep{Frazier2009-knowledge,Wu2016-parallel} is a one-step lookahead BO method that aims to maximize the posterior mean in the next step.
However, computing the AF involves computationally expensive nested optimization.
Therefore, in practice, heuristic approximations are often required.
Furthermore, the analysis of KG is limited to asymptotic convergence \citep{Frazier2009-knowledge} or independent finite-arm settings \citep{Li2022-on}.
Therefore, finite-sample analysis for the usual GP model remains an open problem.

ES \citep{Villemonteix2009-aninformational,Henning2012-Entropy} maximizes the mutual information (MI) between the queried output and the target variable, such as the optimal input, which can be seen as a one-step lookahead minimization of the target variable's entropy.
There are several variants of ES, such as predictive ES (PES) \citep{Hernandez2014-Predictive}, which evaluates the MI regarding the optimal input, max-value ES (MES) \citep{Wang2017-Max}, which evaluates the MI regarding the optimal (maximum) value of the objective function, and joint ES (JES) \citep{hvarfner2022-joint,tu2022-joint}, which evaluates the MI regarding both the optimal input and the optimal value.
However, all of these methods require (i) posterior sampling approximation, (ii) MC approximation, and (iii) heuristic approximation for the conditional distribution given the target variable.
In particular, the validity of the conditional distribution approximation is unknown.
Furthermore, establishing their theoretical guarantees remains an open problem.
\footnote{
Although the regret guarantee for the one-sample variant of MES was established by \citet{takeno2024-posterior}, interpreting it as a one-step lookahead method is not natural.
}

Another closely related one-step lookahead BO method is truncated variance reduction (TruVaR) \citep{bogunovic2016truncated}.
TruVaR minimizes the average (truncated) posterior variance over the set of {\it potential maximizers}, which consists of inputs that could be maximizers given the confidence interval.
Thus, TruVaR and OVR are similar except for the input points whose posterior variances are evaluated.
However, enumerating the potential maximizers for continuous input domains is computationally intractable and the theoretical set of potential maximizers is too conservative in practice.
Furthermore, \citet{bogunovic2016truncated} analyzed TruVaR under the assumption of submodularity of the posterior variance (Assumption 3.1 in \citep{bogunovic2016truncated}), although this submodularity does not hold in general, as discussed in \citep{takeno2025distributionally}.
In contrast, we do not assume such submodularity.

This paper analyzes OVR, which was proposed by \citet{nguyen2017-predictive} under the name predictive variance reduction search (PVRS).
PVRS is a rather generic term that may refer to reducing predictive uncertainty over the entire input space.
We instead use ``Optimal-Point Variance Reduction'' to emphasize that the method targets variance reduction at the posterior-distributed optimal point.
Although \citet{nguyen2017-predictive} proposed OVR as a heuristic, we derive OVR through a one-step optimal policy to minimize the expected posterior variance over the optimal input.
Furthermore, we show the uniform MC estimation error bound and ROVR's simple regret bound.

\section{Preliminary}
\label{sec:preliminary}

\paragraph{Bayesian optimization}
We consider the optimization of a black-box function $f: \cX \rightarrow \RR$:
$$\*x^* = \argmax_{\*x \in \cX} f(\*x),$$
where $\cX \subset \RR^d$ is an input domain and $d$ is the dimension.
We assume that observations $y = f(\*x) + \varepsilon$ are contaminated by noise $\varepsilon$, and the function evaluation cost is high.
Thus, our goal is to optimize $f$ using as few observations as possible.
For this goal, BO sequentially selects an input point $\*x_t$ based on AFs computed by the GP model trained on the currently available dataset and observes $y_t = f(\*x_t) + \varepsilon_t$, where $t$ is the iteration.
%
%
%
Therefore, we can obtain the training data $\cD_{t-1} = \{ (\*x_i, y_{\*x_i}) \}_{i=1}^{t-1}$ by the beginning of the $t$-th iteration.

\paragraph{Regularity assumptions}
We assume the following common regularity assumption \citep{Srinivas2010-Gaussian,Russo2014-learning}:
\begin{assumption}
    The function $f$ follows a zero-mean GP with a predefined kernel $k: \cX \times \cX \rightarrow \RR$, that is, $f \sim \cG \cP (0, k)$, and the $i$-th observation $y_{\*x_i}$ is contaminated by i.i.d. Gaussian noise $\varepsilon_i \sim \cN(0, \sigma^2)$ with a positive variance $\sigma^2 > 0$ as $y_{\*x_i} = f(\*x_i) + \varepsilon_i$.
    In addition, the kernel function satisfies $k(\*x, \*x^\prime) \leq 1$ for all $\*x, \*x^\prime \in \cX$.
    \label{assump:Bayesian}
\end{assumption}
Furthermore, for continuous $\cX$, we assume the following smoothness condition:
\begin{assumption}
    Let $\cX \subset [0, r]^d$ be a compact convex set, where $r > 0$.
    Assume that the kernel $k$ satisfies the following condition on the derivatives of a sample path $f$:
    \begin{align*}
        \Pr \left( \max_{j \in [d]} \sup_{\*x \in \cX} \left| \frac{\partial f(\*u)}{\partial u_j} \Big|_{\*u=\*x} \right| > L \right) \leq a d \exp \left( - \frac{L^2}{b^2} \right),
    \end{align*}
    where $a \geq 1$ and $b > 0$ are absolute constants.
    \label{assump:Bayesian_continuous}
\end{assumption}
This assumption is commonly used \citep{Srinivas2010-Gaussian,Kandasamy2018-Parallelised,Takeno2023-randomized,takeno2024-posterior,iwazaki2025improved} and holds at least for
the squared exponential (SE) kernel $k_{\rm SE} (\*x, \*x^\prime) = \exp\left( - \| \*x - \*x^\prime \|_2^2 / (2 \ell^2) \right)$ and
Mat\'{e}rn-$\nu$ kernels $k_{\rm Mat} = \frac{2^{1 - \nu}}{\Gamma(\nu)} \left( \frac{\sqrt{2\nu} \| \*x - \*x^\prime \|_2 }{\ell} \right)^{\nu} J_{\nu} \left( \frac{\sqrt{2\nu} \| \*x - \*x^\prime \|_2 }{\ell} \right) $ with $\nu > 2$, where $\ell, \nu > 0$ are the lengthscale and smoothness parameter, respectively, and $\Gamma(\cdot)$ and $J_{\nu}$ are Gamma and modified Bessel functions, respectively \citep[Theorem~5 in][]{Ghosal2006-posterior}.
Moreover, for the MC estimation error analysis, we assume the following smoothness condition:
\begin{assumption}
    Let $\cX \subset [0, r]^d$ be a compact set, where $r > 0$.
    Define $\rho (\*u, \*v)^2 = k(\*u, \*u) + k(\*v, \*v) - 2k(\*u, \*v)$.
    There exists a constant $L_k > 0$ such that
    \begin{align*}
        \forall \*u, \*v \in \cX,
        \rho (\*u, \*v) \leq L_k \| \*u - \*v \|_1,
    \end{align*}
    that is, $\rho (\*u, \*v)$ is Lipschitz continuous.
    \label{assump:MCEstimation_continuous}
\end{assumption}
A similar assumption has been used in the GP bandit literature \citep{Chowdhury2017-on,vakili2021-optimal,flynn2026-posterior}.
For example, $L_k = O(1 / \ell)$ for SE and Mat\'ern kernels.

\paragraph{Gaussian process model}
From the assumptions on $f$ and noise $\{ \varepsilon_t \}_{t \in \NN}$ in Assumption~\ref{assump:Bayesian}, the posterior distribution $p(f \mid \cD_{t-1})$ is a GP again~\citep{Rasmussen2005-Gaussian}, whose mean and covariance can be obtained as follows:
\begin{equation}
    \begin{aligned}
        &\mu_{t-1}(\*x) = \*k_{t-1}(\*x)^\top \bigl(\*K_{t-1} + \sigma^2 \*I_{t-1} \bigr)^{-1} \*y_{t-1}, \\
        &k_{t-1}(\*x, \*x^\prime) = k(\*x, \*x^\prime) \\
        &\qquad - \*k_{t-1}(\*x) ^\top \bigl(\*K_{t-1} + \sigma^2 \*I_{t-1} \bigr)^{-1} \*k_{t-1}(\*x^\prime),
    \end{aligned}
    \label{eq:GP}
\end{equation}
where $\*k_{t-1}(\*x) \coloneqq \bigl( k(\*x, \*x_1), \dots, k(\*x, \*x_{t-1}) \bigr)^\top \in \RR^{t-1}$ is the kernel vector, $\*K_{t-1} \in \RR^{(t-1)\times (t-1)}$ is the kernel matrix whose $(i, j)$-element is $k(\*x_i, \*x_j)$, $\*I_{t-1} \in \RR^{(t-1)\times (t-1)}$ is the identity matrix, and $\*y_{t-1} \coloneqq (y_1, \dots, y_{t-1})^\top \in \RR^{t-1}$.
%
%
For the sake of notational simplicity, let $\sigma_{t-1}^2 (\*x) = k_{t-1}(\*x, \*x)$ and the posterior variance when $\*x_t = \*x$ be $\sigma_{t}^2 (\*x^\prime \mid \*x) \coloneqq \sigma_{t-1}^2 (\*x^\prime) - \frac{k_{t-1}^2(\*x, \*x^\prime)}{\sigma_{t-1}^2 (\*x) + \sigma^2}$.
Note that the posterior variance calculation does not require $\*y_t$.

\paragraph{Maximum information gain}
For regret analysis, we will use the quantity called maximum information gain (MIG) \citep{Srinivas2010-Gaussian,vakili2021-information,iwazaki2025improved,iwazaki2026tighter}:
\begin{definition}[Maximum information gain]
    Let $f \sim \cG \cP (0, k)$ over $\cX \subset [0, r]^d$.
    Let $A = \{ \*a_i \}_{i=1}^T$, where $\*a_i \in \cX$ for all $i \in [T]$.
    Let $\*f_A = \bigl(f(\*a_i) \bigr)_{i=1}^T$, $\*\varepsilon_A = \bigl(\varepsilon_i \bigr)_{i=1}^T$, where $\varepsilon_i \sim \cN(0, \sigma^2)$ for all $i \in [T]$, and $\*y_A = \*f_A + \*\varepsilon_A \in \RR^T$.
    Then, MIG $\gamma_T$ is defined as follows:
    \begin{align}
        \begin{aligned}
            &\gamma_T \coloneqq \sup_{\forall i \in [T] \*a_i \in \cX} I(\*y_A ; \*f_A),
        \end{aligned}
    \end{align}
    where $I$ is the Shannon mutual information.
    \label{def:MIG}
\end{definition}
For frequently used kernels, MIG is known to be sublinear \citep{iwazaki2025improved,iwazaki2026tighter}, for example,
$\gamma_T = O\bigl( d \log T \bigr)$ for the linear kernels,
$\gamma_T = O\bigl( (\log T)^{d+1} \bigr)$ for the SE kernels, and
$\gamma_T = \widetilde{O}\bigl( T^{\frac{d}{2\nu + d}} \bigr)$ for the Mat\'ern-$\nu$ kernels, where $\widetilde{O}$ hides polylogarithmic factors.

\paragraph{Performance measure}
We evaluate the theoretical performance of BO methods by the expected Bayesian simple regret~\citep{Russo2014-learning,russo2018learning,Kandasamy2018-Parallelised,paria2020-flexible,Takeno2023-randomized,takeno2024-posterior,takeno2025regretEI,takeno2025-regret} defined as
\begin{align}
    {\rm BSR}_T = \EE \left[ f(\*x^*) - f(\hat{\*x}_T) \right],
\end{align}
where we define $\hat{\*x}_T = \argmax_{\*x \in \cX} \mu_T (\*x)$ and the expectation is taken over all sources of randomness, that is, $f, \{\varepsilon_i\}_{i \in \NN}$, and the randomness of the algorithm.
Note that, although the expected regret quantifies the average performance of the algorithm, the expected regret upper bound immediately implies a high-probability upper bound via Markov's inequality \citep[Sec. 3.1 in][]{Russo2014-learning}.
Our goal is to show the convergence rate of the BSR.

\section{OVR and regularized OVR}
\label{sec:algorithm}

This section provides the definitions of OVR and ROVR, which mainly aim to minimize $\EE \sbr{\sigma_T(\*x^*) \mid \cD_T}$.
Although our formulation accounts for the posterior STD, we use the term {\it variance reduction} since it is more common.
On the other hand, even if we consider minimizing $\EE \sbr{\sigma_T^2(\*x^*) \mid \cD_T}$, our theoretical analyses hold in almost the same way.
Algorithm~\ref{alg:OVR} presents the pseudocode for OVR and ROVR.

We motivate minimization of $\EE \sbr{\sigma_T(\*x^*) \mid \cD_T}$ from the perspective of regret analysis.
From the analysis by \citet{Russo2014-learning} (see Appendix~\ref{sec:proof_regret} for details), we can obtain
\begin{align*}
    \EE[f(\*x^*) - f(\hat{\*x}_T)]
    &\leq \EE[f(\*x^*) - \mu_T (\*x^*)] \\
    &\leq \beta_T^{1/2} \EE[\sigma_T (\*x^*)] + O \rbr{\frac{1 + \sqrt{\beta_T}}{T}},
\end{align*}
where $\beta_T = O\rbr{\log (e|\cX|T)}$ if $|\cX| < \infty$ or $\beta_T = O\rbr{d\log (eT)}$ under Assumptions~\ref{assump:Bayesian_continuous} and \ref{assump:MCEstimation_continuous}.
Here, the expectation is taken over all sources of randomness, including the observed data and the algorithm, and $\*x^*$ is defined via $f$.
Therefore, decreasing $\EE[\sigma_T (\*x^*) \mid \cD_T]$ on average over $\cD_T$ leads to a tighter BSR bound.
However, optimizing $\{\*x_1, \dots, \*x_T \}$ for minimizing $\EE[\sigma_T (\*x^*) \mid \cD_T]$ is computationally intractable.
Hence, we develop the AF to minimize $\EE[\sigma_t (\*x^* \mid \*x) \mid \cD_{t-1}]$ in a one-step lookahead manner and its regularized variant, which achieves the vanishing BSR upper bound.

\begin{algorithm}[!t]
    \caption{OVR and ROVR}\label{alg:OVR}
    \begin{algorithmic}[1]
        \Require Domain $\cX$, kernel $k$, initial dataset $\cD_{0}$, \# MC samples $\{M_t\}_{t \in \NN}$, regularization parameter $\{ c_t \}_{t \in \NN}$ ($c_t = 0$ for all $t \in \NN$ for OVR)
        \For{$t = 1, \dots, T$}
            \State Update $\mu_{t-1} (\cdot)$ and $\sigma_{t-1}^2 (\cdot)$ by Eq.~\eqref{eq:GP}
            \For{$m = 1, \dots, M_t$}
                \State Generate $g_t \sim p(f \mid \cD_{t-1})$
                \State Optimize $\*x^*_{t, m} = \argmax_{\*x \in \cX} g_t(\*x)$
            \EndFor
            \State Compute $\*x_t = \argmin_{\*x \in \cX} \Bigl\{\frac{1}{M_t} \sum_{m = 1}^{M_t} \sigma_t (\*x^*_{t, m} \mid \*x) - c_t \sigma_{t-1}(\*x) \Bigr\}$
            \State Observe $y_{t}$ and set $\cD_{t} = \cD_{t-1} \cup \{ (\*x_t, y_t) \}$
        \EndFor
        \State \Return $\hat{\*x}_T = \argmax_{\*x \in \cX} \mu_T (\*x)$
    \end{algorithmic}
\end{algorithm}

\subsection{AF of OVR}
\label{sec:OVR}

From the discussion above, we define the AF of OVR, which is minimized over $\cX$, as follows:
\begin{align*}
    \alpha_t (\*x)
    &= \EE_{y_{\*x}} \Bigl[ \EE_{\*x^*} \bigl[\sigma_t (\*x^* \mid \*x) \bigm| \widetilde{\cD}_{t, \*x}\bigr] \bigm| \cD_{t-1}, \*x \Bigr],
\end{align*}
where we define $\widetilde{\cD}_{t, \*x} = \cD_{t-1} \cup \{ (\*x, y_{\*x}) \}$.
%
%
In contrast to ES-based methods, which rely on heuristic approximations discussed in Section~\ref{sec:related}, we can transform the AF analytically as
\begin{align*}
    \alpha_t (\*x)
    &= \EE_{y_{\*x}} \Bigl[ \EE_{\*x^*} \bigl[\sigma_t (\*x^* \mid \*x) \bigm| \widetilde{\cD}_{t, \*x}\bigr] \bigm| \cD_{t-1}, \*x \Bigr] \\
    &= \EE_{y_{\*x}, \*x^*} \sbr{\sigma_t (\*x^* \mid \*x) \bigm| \cD_{t-1},\*x} \\
    &= \EE_{\*x^*} \sbr{\sigma_t (\*x^* \mid \*x) \bigm| \cD_{t-1}},
\end{align*}
where we use the tower property of the expectation and the fact that $\sigma_t (\*x^* \mid \*x)$ does not depend on $y_{\*x}$.
Therefore, although the expectation over $p(\*x^* \mid \cD_{t-1})$ cannot be computed analytically as with ES \citep{Villemonteix2009-aninformational,Henning2012-Entropy} and PES \citep{Hernandez2014-Predictive}, we do not require any other approximation and minimizing $\alpha_t (\*x)$ yields a one-step optimal policy for minimizing $\EE_{\*x^*} \bigl[\sigma_t (\*x^*) \bigm| \cD_t\bigr]$.

We apply MC estimation to approximate the expectation over $p(\*x^* \mid \cD_{t-1})$ similarly to ES-based methods \citep{Hernandez2014-Predictive,Wang2017-Max,hvarfner2022-joint}.
Therefore, in practice, we use the following approximate AF:
\begin{align*}
    \hat{\alpha}_t (\*x)
    &= \frac{1}{M_t} \sum_{m = 1}^{M_t} \sigma_t (\*x^*_{t, m} \mid \*x),
\end{align*}
where $M_t \in \NN$ is a user-specified parameter and $\{ \*x^*_{t, m} \}_{m=1}^{M_t}$ consists of i.i.d. samples from the posterior $p(\*x^* \mid \cD_{t-1})$.
%
%
Since sampling from $p(\*x^* \mid \cD_{t-1})$ is computationally intractable, particularly when $\cX$ is continuous, we employ the approximate posterior sampling method as in prior works \citep{Hernandez2014-Predictive,Wang2017-Max,hvarfner2022-joint}.
Specifically, we employ the method of \citet{Wilson2020-efficiently,wilson2021pathwise}, which combines pathwise conditioning and the random Fourier feature approximation \citep{Rahimi2008-Random}.
Thus, by generating (approximate) posterior sample paths and optimizing them, we can obtain $\{ \*x^*_{t, m} \}_{m=1}^{M_t}$.


\subsection{AF of Regularized OVR}
\label{sec:ROVR}

This section describes ROVR, which is a generalization of OVR that incorporates regularization with respect to the posterior STD.
OVR is a one-step optimal algorithm with respect to minimizing $\EE_{\*x^*} \bigl[\sigma_t (\*x^*) \bigm| \cD_t\bigr]$.
However, the posterior variance and STD are not submodular in general, as discussed in \citep{takeno2025distributionally}.
Therefore, we cannot obtain multi-step optimality with respect to $\{\*x_1, \dots, \*x_T\}$ from the submodular property.
Thus, we derive the theoretical guarantee through a slight modification that facilitates exploration.

Specifically, we consider the following variant of OVR:
\begin{align*}
    \alpha_t (\*x)
    &= \EE_{\*x^*} \sbr{\sigma_t (\*x^* \mid \*x) \bigm| \cD_{t-1}} - c_t \sigma_{t-1}(\*x), \\
    \hat{\alpha}_t (\*x)
    &= \frac{1}{M_t} \sum_{m = 1}^{M_t} \sigma_t (\*x^*_{t, m} \mid \*x) - c_t \sigma_{t-1}(\*x),
\end{align*}
where $c_t \geq 0$ is the regularization parameter, whose theoretical choice will be shown in Section~\ref{sec:analysis_regret}.
Obviously, setting $c_t = 0$ recovers the usual OVR.
If $c_t > 0$, the AF is smaller if $\sigma_{t-1}(\*x)$ is larger.
Therefore, since we minimize the AF, this modification promotes selecting an input with high uncertainty.

\subsection{Similarity to Existing BO Methods}
\label{sec:alg_comparison}

\paragraph{Similarity to TS}
When $M_t = 1$ for all $t \in \NN$ and the noise variance $\sigma^2 = 0$, OVR is the same as TS since OVR must select $\*x_t = \*x^*_{t, 1}$.
Therefore, in such a case, OVR can achieve the expected BSR and Bayesian cumulative regret bounds by combining the analyses in \citep{russo2018learning,takeno2024-posterior,takeno2026regret} and \citep{iwazaki2025gaussian}.
For sufficiently increasing $M_t$, we will show the vanishing BSR bound for ROVR in Section~\ref{sec:analysis_regret}.
On the other hand, TS often suffers from over-exploration \citep{Shahriari2016-Taking,takeno2024-posterior}.
OVR alleviates this limitation of TS by using a large number $M_t$ of MC samples in practice.

\paragraph{Similarity to JES}
Moreover, OVR is also similar to JES \citep{hvarfner2022-joint}.
Roughly speaking, although OVR minimizes $\EE_{\*x^*} \sbr{\sigma_t (\*x^* \mid \*x) \bigm| \cD_{t-1}}$, JES maximizes $\alpha_{\rm JES}(\*x) \approx \log \rbr{\sigma_{t-1}^2 (\*x) + \sigma^2} - \EE_{\*x^*, f^*} \sbr{\log \rbr{\sigma_t^2 \bigl(\*x \mid (\*x^*, f^*), f(\*x) \leq f^* \bigr) + \sigma^2} \bigm| \cD_{t-1}}$, where the logarithm comes from the definition of entropy and $\sigma_t^2 \bigl(\*x \mid (\*x^*, f^*), f(\*x) \leq f^* \bigr)$ is the approximate variance given $\*x^*$, $f^* = f(\*x^*)$, and the truncation condition $f(\*x) \leq f^*$.
From this difference, OVR has three advantages over JES.
First, the AF value of JES can approach infinity $\max_{\*x \in \cX} \alpha_{\rm JES}(\*x) \rightarrow \infty$ as $\sigma^2 \rightarrow 0$, since the argument of the logarithm approaches $0$.
Thus, when $\sigma^2$ is sufficiently small, JES can cause over-exploration similarly to TS.
Second, the validity of the heuristic variance approximation for $\sigma_t^2 \bigl(\*x \mid (\*x^*, f^*), f(\*x) \leq f^* \bigr)$ is unknown.
Third, the regret guarantee for ES-based BO methods is still an open problem.
In contrast, OVR alleviates over-exploration by relying on the posterior STD and does not require a heuristic conditional variance approximation.
In addition, we show that ROVR achieves the vanishing BSR upper bound.

\section{Analysis}
\label{sec:analysis}

This section presents analyses of the MC estimation error and the BSR upper bound.
In this section, we ignore the error of the posterior sampling approximation, as with existing analyses \citep{Russo2014-learning,Kandasamy2018-Parallelised,takeno2024-posterior,takeno2025regretEI}.
This approximation error may be handled, e.g., by the analyses by \citet{Mutny2018-efficient,vakili2021-scalable}.

\subsection{Analysis of Monte Carlo Estimation}
\label{sec:analysis_MC}

The MC estimation error for fixed $\*x \in \cX$ can be obtained via a classical concentration inequality.
Since $\sigma_t(\*x^* \mid \*x) \in [0, 1]$ is a bounded random variable, we can apply Hoeffding's inequality \citep[Theorem 2.2.6 in][]{Vershynin2018-high} shown in Theorem~\ref{thm:hoeffding} in the appendix and can obtain
\begin{align*}
    &\Pr\Biggl( |\alpha_t (\*x) - \hat{\alpha}_t (\*x)| \leq \sqrt{\frac{\log(2 / \delta)}{2M_t}} \biggm| \cD_{t-1}\Biggr)
    \geq 1 - \delta,
\end{align*}
for any $M_t \in \NN$, $\*x \in \cX$, $\cD_{t-1}$, $\delta \in (0, 1)$, and $c_t \geq 0$.
Therefore, we can obtain the $O(1 / \sqrt{M_t})$ bound, which also achieves logarithmic dependence on the probability $\delta$.
However, in practice, we need to obtain an error bound that holds uniformly for all $\*x \in \cX$ since we compute AF values for many $\*x \in \cX$ to optimize the AF.
For both discrete and continuous domains, we show such uniform error bounds.

If the input domain $\cX$ is finite, we can obtain the following theorem by the union bound:
\begin{restatable}[MC estimation error bound for discrete domains]{theorem}{thmMCErrorDiscrete}
    \label{thm:MCError_discrete}
    Pick $\delta \in (0, 1)$.
    Assume $\cX$ is finite.
    Then, for any $\cD_{t-1}$, $c_t \geq 0$, and $M_t \in \NN$, we have
    \begin{align*}
        \Pr\Biggl( \max_{\*x \in \cX} |\alpha_t(\*x) - \hat{\alpha}_t(\*x)| \! \leq \! \sqrt{\frac{\log \rbr{\frac{2|\cX|}{\delta} }}{2M_t}} \biggm| \! \cD_{t-1} \! \Biggr) \!
        \geq 1 - \delta.
    \end{align*}
\end{restatable}
See Appendix~\ref{sec:proof_MC_discrete} for the proof.
Therefore, to guarantee the $\epsilon$-accuracy, that is, $\max_{\*x \in \cX} |\alpha(\*x) - \hat{\alpha}(\*x)| \leq \epsilon$ with some $\epsilon > 0$, we need $M_t = \Omega\rbr{1 / \epsilon^2}$.

If the input domain $\cX$ is continuous, we cannot apply the union bound so that the error bounds hold for all $\*x \in \cX$.
Thus, to obtain the uniform error bound, we need some smoothness condition.
To derive such a condition, we derive the following lemma:
\begin{restatable}[Lipschitz continuity of posterior conditional STD]{lemma}{lemLipschitzContinuityPosteriorConditionalSTD} \label{lemLipschitzContinuityPosteriorConditionalSTD}
    Suppose that Assumption~\ref{assump:MCEstimation_continuous} holds.
    Then, for any $t \in \NN$ and $\cD_{t-1}$, the following Lipschitz continuity holds:
    \begin{align*}
        \forall \*x, \*u, \*v \in \cX,
        |\sigma_t(\*x \mid \*u) - \sigma_t(\*x \mid \*v)| \leq \frac{L_k}{2 \sigma} \| \*u - \*v \|_1,
    \end{align*}
    where $L_k$ is defined as in Assumption~\ref{assump:MCEstimation_continuous}.
\end{restatable}
See Appendix~\ref{sec:proof_MC_continuous} for the proof.

From Lemma~\ref{lemLipschitzContinuityPosteriorConditionalSTD}, we can obtain the following result for continuous domains:
\begin{restatable}[Informal MC estimation error bound for continuous domains]{theorem}{thmMCErrorContinuous}
    \label{thm:MCError_continuous}
    Pick $\delta \in (0, 1)$.
    Let $\overline{E_t} = \max_{\*x \in \cX} |\alpha_t(\*x) - \hat{\alpha}_t(\*x)|$.
    Suppose that Assumption~\ref{assump:MCEstimation_continuous} holds.
    Then, we have
    \begin{align*}
        &\Pr \rbr{ \overline{E_t} = O\rbr{\sqrt{\frac{\log \rbr{M_t / \delta}}{M_t}}} \biggm| \cD_{t-1} } \geq 1 - \delta, \\
        &\EE \sbr{ \overline{E_t} \bigm| \cD_{t-1}} = O\rbr{\sqrt{\frac{\log(M_t)}{M_t}}},
    \end{align*}
    for any $\cD_{t-1}$, $M_t \in \NN$, and $c_t \geq 0$.
\end{restatable}
See Appendix~\ref{sec:proof_MC_continuous} for the proof and the formal version.
Thus, also for continuous input domains, we obtained the uniform error bound, which implies that we need $M_t = \Omega\rbr{\log(1 / \epsilon) / \epsilon^2}$ for the $\epsilon$-accuracy with $\epsilon \in (0, 1)$.

\subsection{Regret Analysis}
\label{sec:analysis_regret}

For the regret analysis, we leverage the following lemma to handle continuous input domains:
\begin{restatable}[Lipschitz continuity of posterior STD]{lemma}{LemLipschitzPosteriorSTD}
    \label{LemLipschitzPosteriorSTD}
    Suppose that Assumptions~\ref{assump:Bayesian} and \ref{assump:MCEstimation_continuous} hold.
    Then, for any $t \in \NN$, $\cD_{t-1}$, and $\*u, \*v \in \cX$, we have
    \begin{align*}
        |\sigma_{t-1}(\*u) - \sigma_{t-1}(\*v)|
        \leq L_k \| \*u - \*v \|_1.
    \end{align*}
\end{restatable}
See Appendix~\ref{sec:proof_regret} for the proof.
Similar results are shown in Lemma G.3 of \citep{flynn2026-posterior} and Theorem E.4 of \citep{Kusakawa2022-bayesian}.

Our ROVR achieves the following BSR bound:
\begin{restatable}[BSR bound]{theorem}{thmBSRBound} \label{thmBSRBound}
    Suppose that, in addition to Assumption~\ref{assump:Bayesian}, $|\cX| < \infty$ or Assumptions~\ref{assump:Bayesian_continuous} and \ref{assump:MCEstimation_continuous} hold.
    Then, if Algorithm~\ref{alg:OVR} runs with a non-increasing sequence $\{ c_t \}_{t \in \NN}$ such that $c_t > 0$ for all $t \in \NN$, the following bound holds:
    \begin{align*}
        {\rm BSR}_T &\leq \sqrt{ \frac{C_1 \beta_T \gamma_T}{T}} + O\rbr{\frac{\sqrt{\beta_T}}{T c_T}} + O \rbr{\frac{\sqrt{\beta_T}}{T} \sum_{t=1}^T \frac{\sqrt{\log(M_t)}}{\sqrt{M_t} c_t} },
    \end{align*}
    where $C_1 = 2 / \log \rbr{1 + \sigma^{-2}}$ and $\beta_T = 2 \log(1 + |\cX|T)$ if $|\cX| < \infty$ or $\beta_T = 2 \log\rbr{1 + (1 + drTL)^d T}$ otherwise, where $L = \max\cbr{L_k, b\rbr{\sqrt{\log(ad)} + \sqrt{\pi}/2}}$ and $L_k$ is defined in Assumption~\ref{assump:MCEstimation_continuous}.
\end{restatable}
See Appendix~\ref{sec:proof_regret} for the detailed proof, though we show a proof sketch below.

\paragraph{Proof sketch}
For simplicity, we consider the case $|\cX| < \infty$ here.
By the standard arguments from \citet{Russo2014-learning}, we can see that
\begin{align*}
    {\rm BSR}_T
    &\leq \beta_T^{1/2} \EE \sbr{\sigma_{T}(\*x^*)} + \frac{1}{T},
\end{align*}
where $\beta_T = O(\log(|\cX|T))$.
Then, from $\EE [\sigma_1(\*x^*)] \geq \dots \geq \EE [\sigma_t(\*x^*)] \dots \geq \EE [\sigma_T(\*x^*)]$, which is also used in \citep{nava2022diversified}, we have
\begin{align*}
    {\rm BSR}_T
    &\leq \frac{\beta_T^{1/2}}{T} \sum_{t=1}^T \EE \sbr{\sigma_{t-1}(\*x^*)} + \frac{1}{T}.
\end{align*}
From the definition of ROVR's AF, for all $t \in [T]$, we can obtain
\begin{align*}
    \EE \sbr{\sigma_{t-1}(\*x^*)}
    &\leq \EE\sbr{\sigma_{t-1}(\*x_t)} + O\rbr{\frac{\sqrt{\log(M_t)}}{\sqrt{M_t} c_t}} + \frac{1}{c_t}\rbr{\EE \sbr{ \sigma_{t-1}(\*x^*)} - \EE \sbr{ \sigma_{t}(\*x^*)}},
\end{align*}
where the second term on the right-hand side implies the MC estimation error bounds shown in Theorems~\ref{thm:MCError_discrete} and \ref{thm:MCError_continuous}.
Finally, by bounding $\sum_{t=1}^T \frac{1}{c_t}\rbr{\EE \sbr{ \sigma_{t-1}(\*x^*)} - \EE \sbr{ \sigma_{t}(\*x^*)}}$ as a telescoping series and applying Lemma 5.4 in \citep{Srinivas2010-Gaussian} to $\EE\sbr{\sum_{t=1}^T \sigma_{t-1}(\*x_t) }$, we can obtain the desired result.

\begin{figure*}[t]
    \centering
    \includegraphics[width=0.9\linewidth]{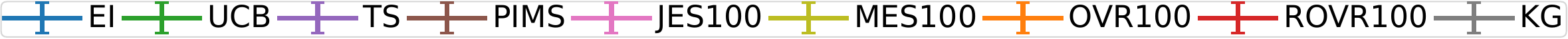}

    \includegraphics[width=0.021\linewidth]{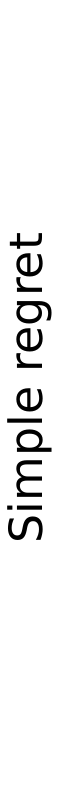}
    \includegraphics[width=0.238\linewidth]{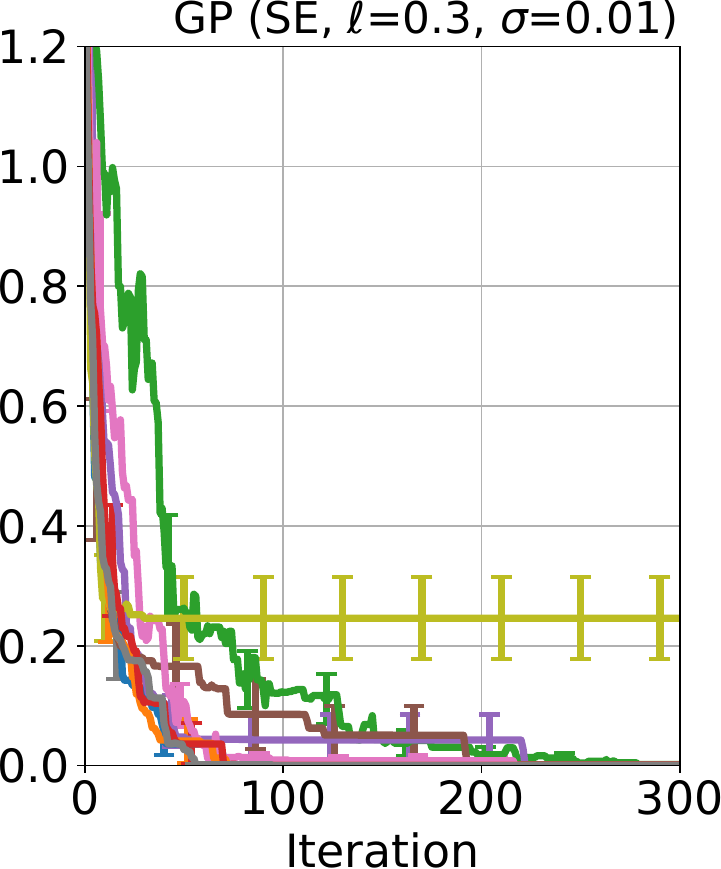}
    \includegraphics[width=0.238\linewidth]{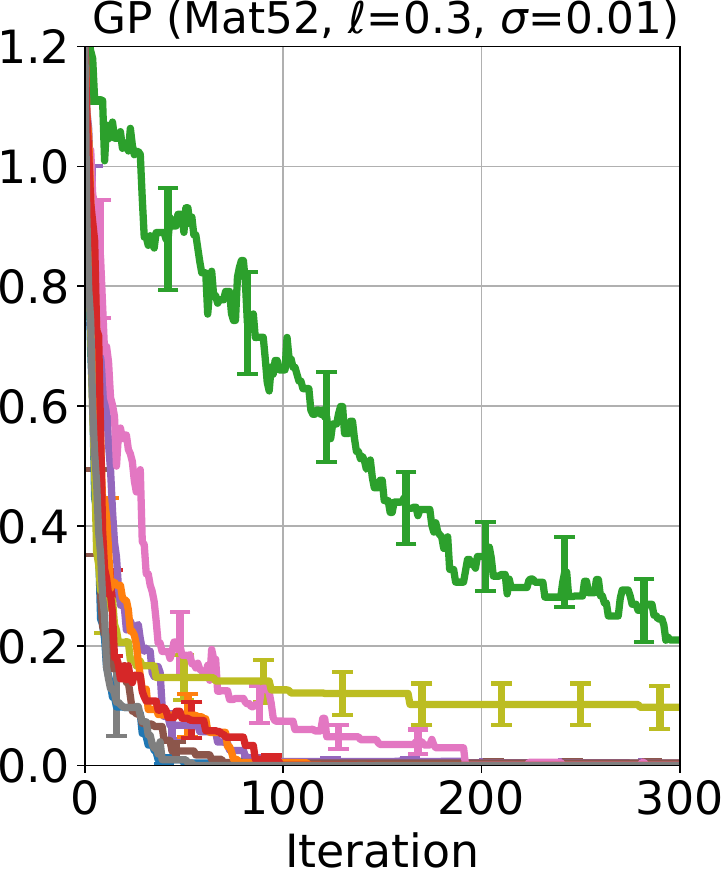}
    \includegraphics[width=0.238\linewidth]{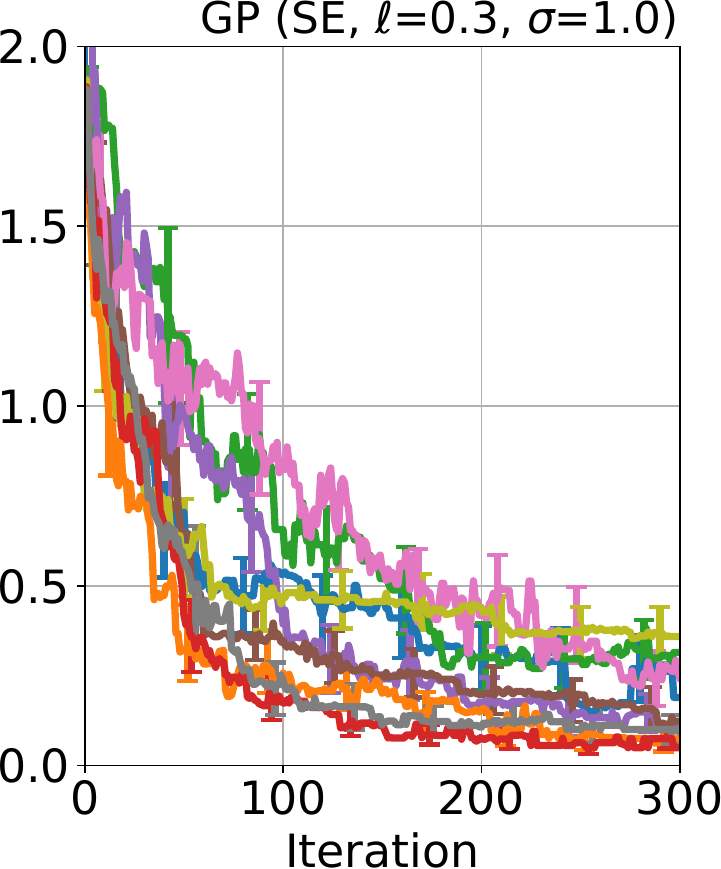}
    \includegraphics[width=0.238\linewidth]{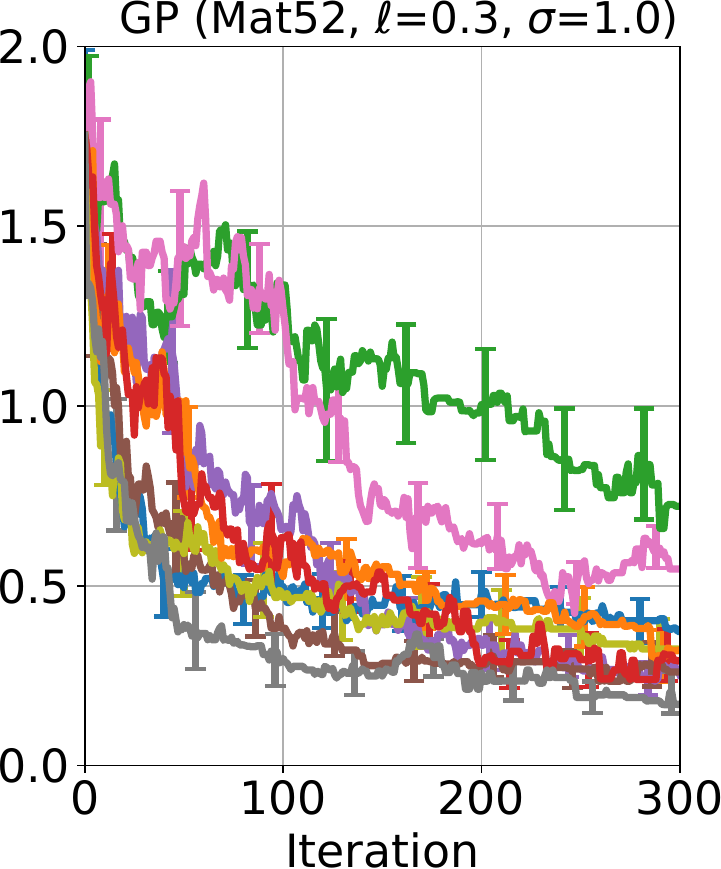}

    \caption{
    The average and standard error of simple regret in the synthetic function experiments.
    The numerical suffix in the legend indicates the number of MC samples.
    The left two plots and the right two plots show the results for $\sigma^2=10^{-4}$ and $\sigma^2=1$, respectively.
    }
    \label{fig:exp_synthetic}
\end{figure*}

\paragraph{Convergence of our BSR bound}
We can choose $M_t$ and $c_t$ so that all terms in the BSR upper bound converge to $0$.
For example, let $M_t = \Theta(t)$ and let $c_t > 0$ be non-increasing with $c_t = \Omega\rbr{1 / (\log(t))^{d / 2}}$.
Then, $1 / (T c_T) = O\rbr{(\log(T))^{d/2} / T}$ and
\begin{align*}
    \frac{1}{T} \sum_{t=1}^T \frac{\sqrt{\log(M_t)}}{\sqrt{M_t} c_t}
    = O\rbr{\frac{(\log(T))^{(d+1)/2}}{\sqrt{T}}}.
\end{align*}
As discussed in Section~\ref{sec:preliminary}, the MIG for commonly used kernel functions satisfies, e.g., $\gamma_T = O\rbr{(\log T)^{d+1}}$ for SE kernels and $\gamma_T = \widetilde{O}\rbr{T^{\frac{d}{2\nu + d}}}$ for Mat\'ern kernels.
Since $\beta_T = O(\log(T))$ for fixed $d$, the MC error contribution is $O\rbr{(\log(T))^{(d+2)/2} / \sqrt{T}}$.
Consequently, the resulting BSR bounds are $O\rbr{(\log(T))^{(d+2)/2} / \sqrt{T}}$ for SE kernels and $\widetilde{O}\rbr{T^{-\nu/(2\nu+d)}}$ for Mat\'ern kernels.
Both bounds converge to $0$; for SE kernels, the MC error contribution has the same order as the bound obtained from the MIG term.

\paragraph{Comparison with existing bounds}
Several studies show a tighter $\widetilde{O}(\sqrt{T})$ (high-probability) Bayesian cumulative regret (BCR) bound \citep{iwazaki2025improved,takeno2026regret}.
However, \citet{iwazaki2025improved,takeno2026regret} stated that showing $O(\sqrt{T})$ expected BCR bounds is an open problem.
In addition, our BSR bound rate matches those of most existing expected BSR bounds \citep{Russo2014-learning,Kandasamy2018-Parallelised,Takeno2023-randomized,takeno2024-posterior,takeno2025-regret,takeno2025regretEI}.
We believe our analysis is an important step toward obtaining regret guarantees for one-step lookahead BO.

Showing upper bounds on cumulative regret for ROVR is not straightforward.
Given that ROVR aims to minimize the posterior STD in the next step, it is difficult to obtain an instantaneous regret bound, unlike bandit-based BO algorithms, such as UCB and TS.
We believe that this is a common difficulty for the one-step lookahead AFs.
Thus, developing one-step lookahead BO methods with sublinear cumulative regret bounds is an intriguing direction for future work.

\begin{figure*}[t]
    \centering
    \includegraphics[width=0.9\linewidth]{fig/legend.pdf}

    \includegraphics[width=0.021\linewidth]{fig/vertical_axes_label.pdf}
    \includegraphics[width=0.238\linewidth]{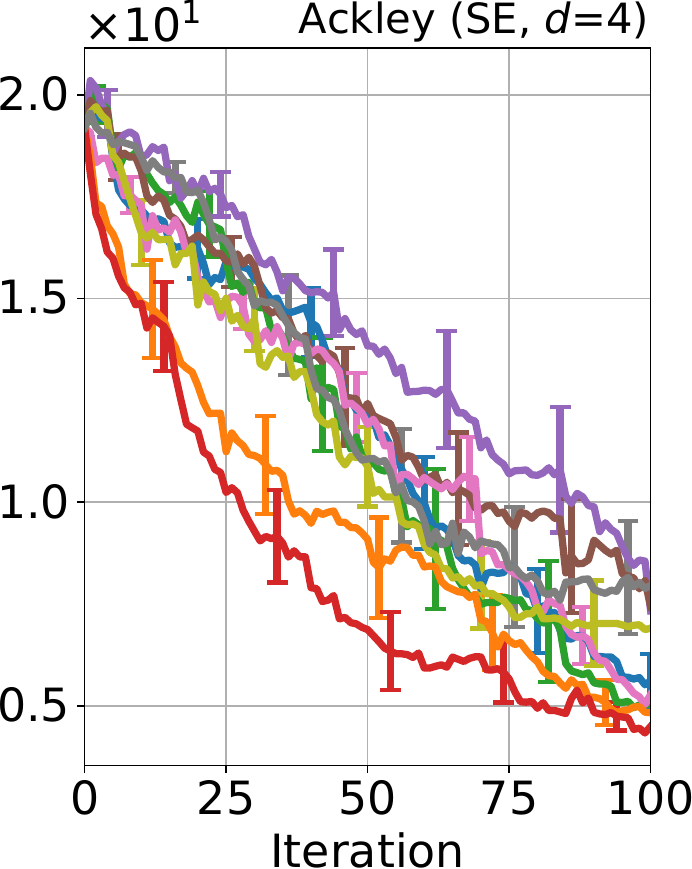}
    \includegraphics[width=0.238\linewidth]{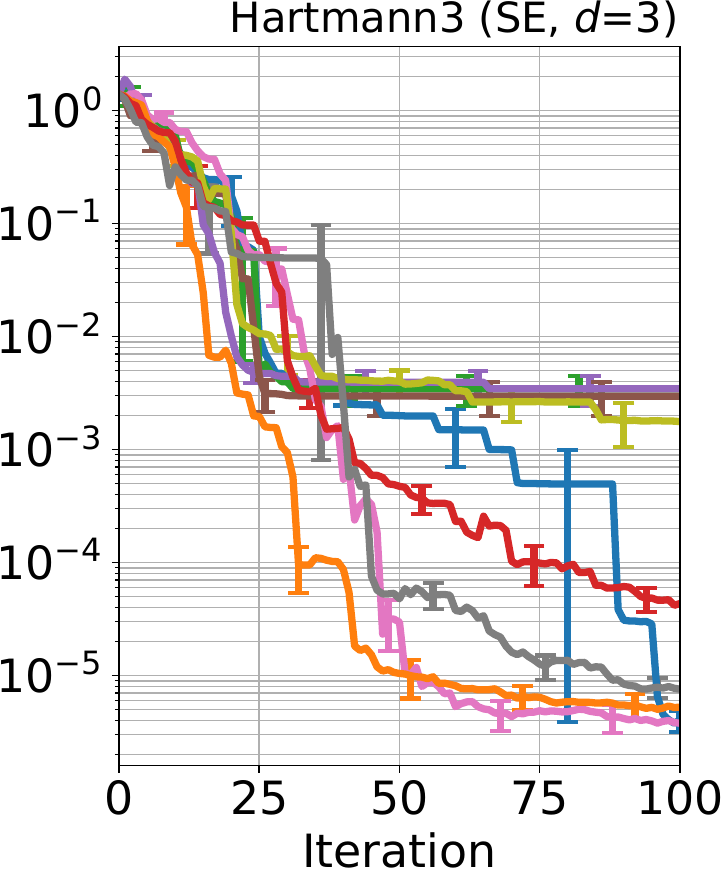}
    \includegraphics[width=0.238\linewidth]{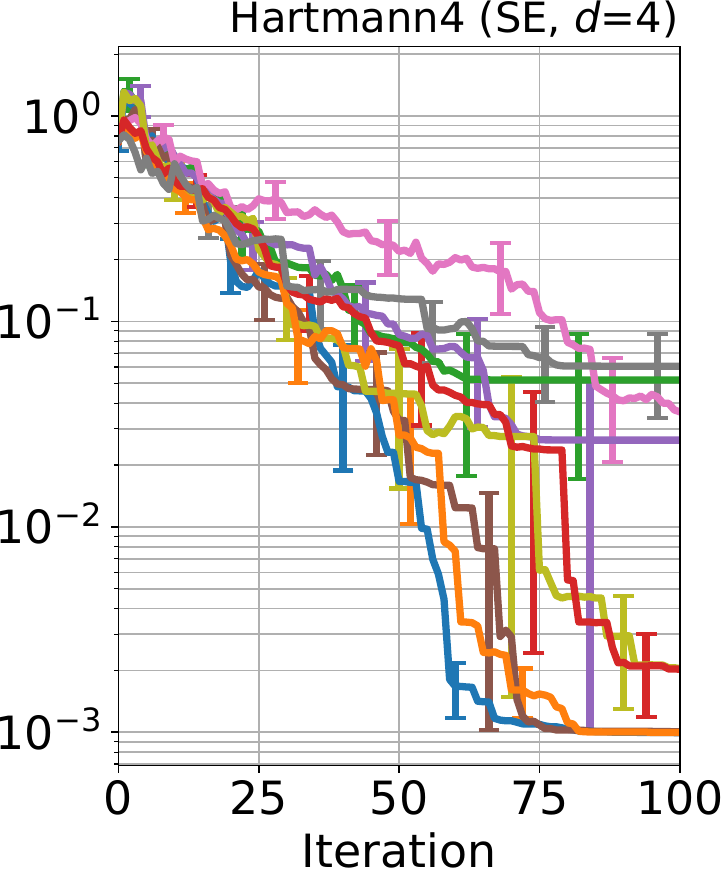}
    \includegraphics[width=0.238\linewidth]{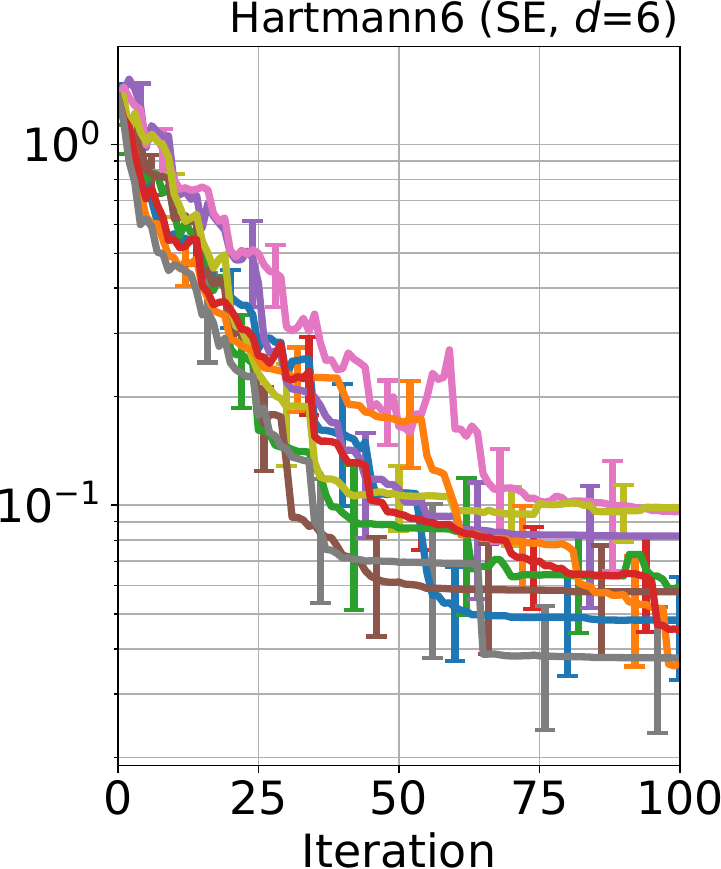}

    \includegraphics[width=0.021\linewidth]{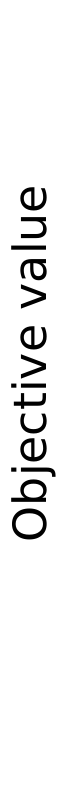}
    \includegraphics[width=0.238\linewidth]{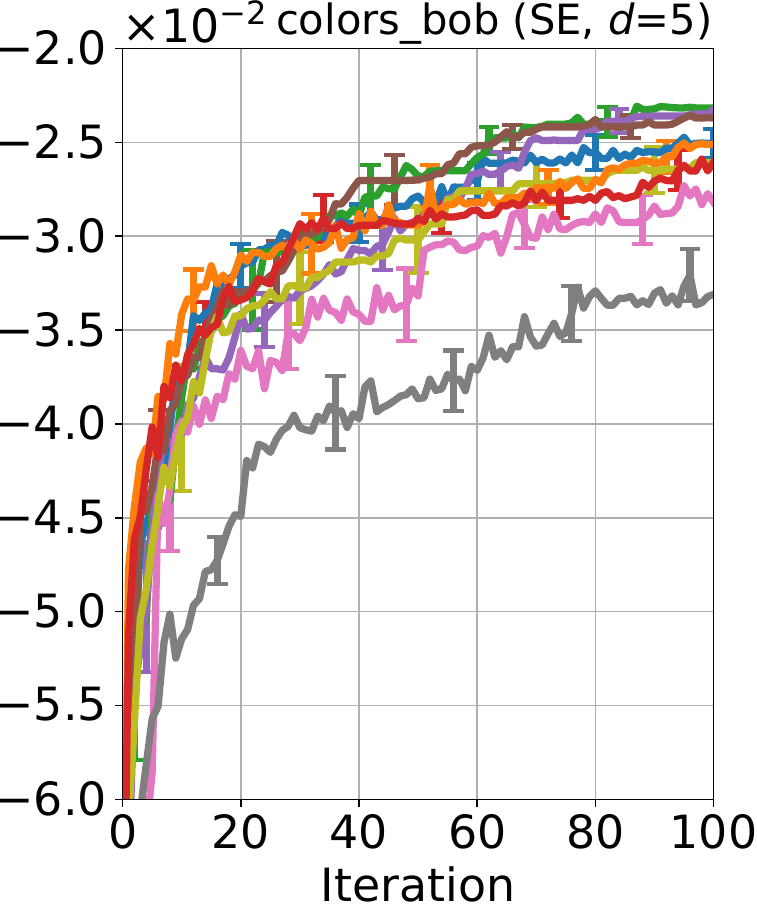}
    \includegraphics[width=0.238\linewidth]{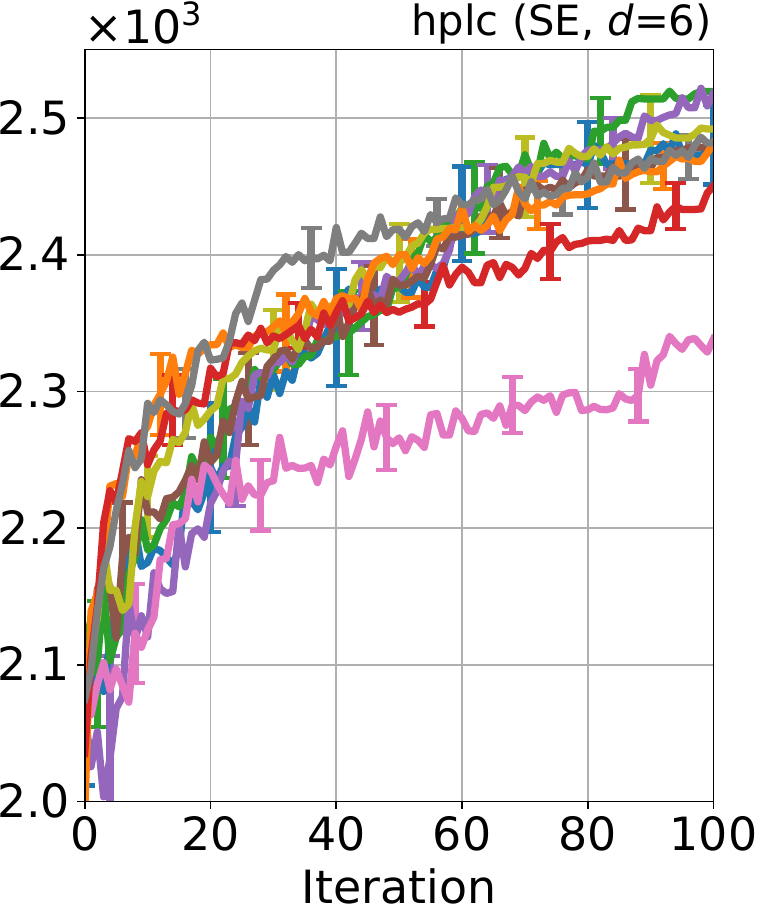}
    \includegraphics[width=0.238\linewidth]{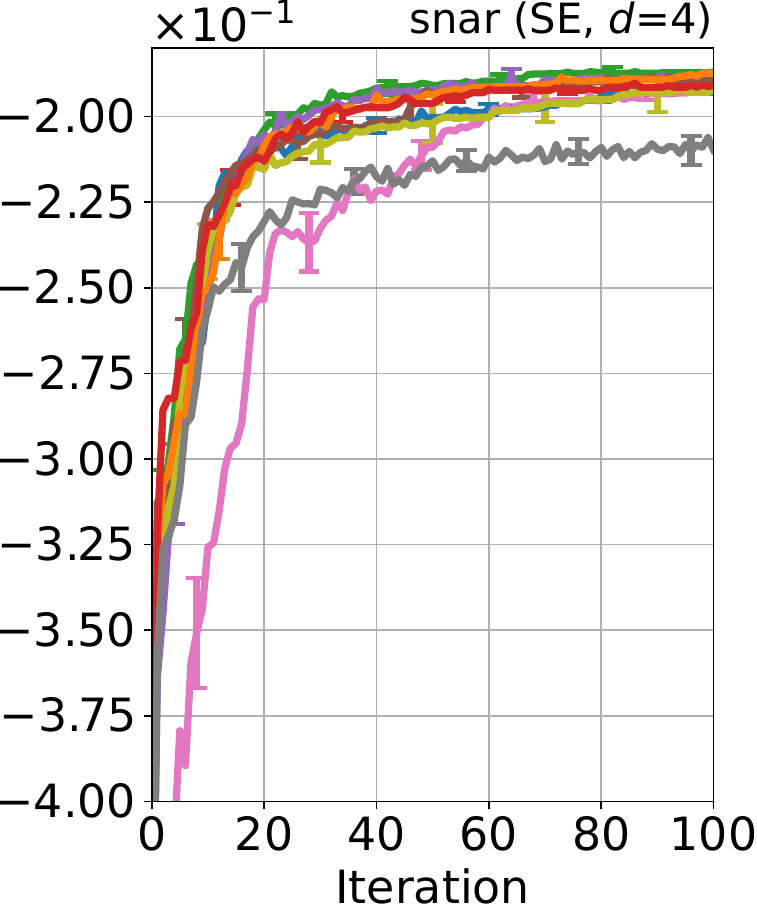}
    \includegraphics[width=0.238\linewidth]{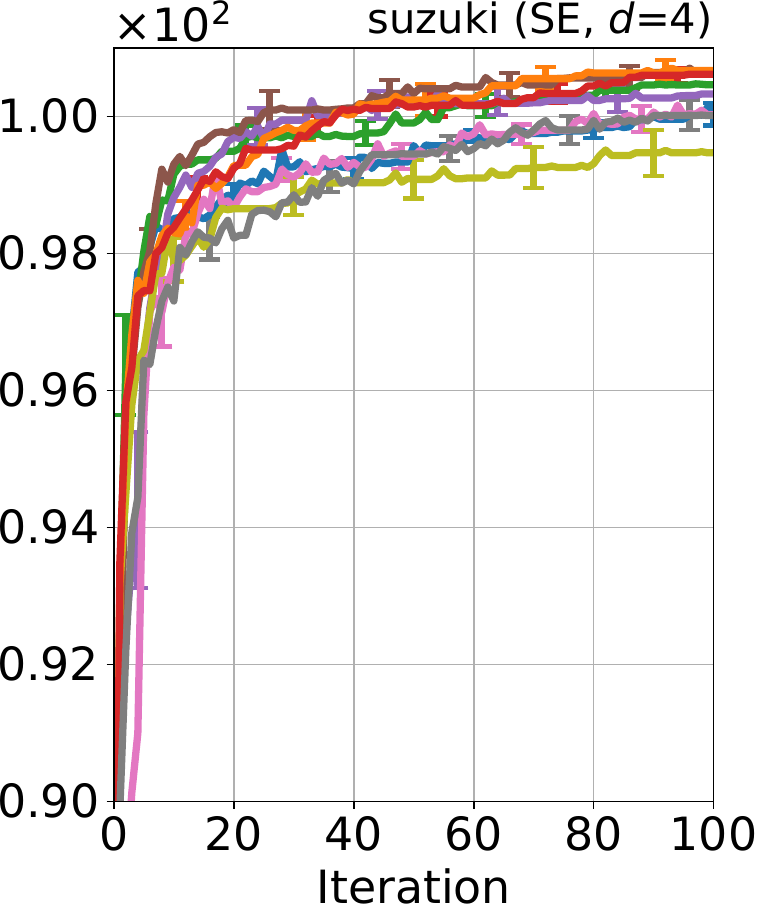}
    \caption{
    The average and standard error of experimental results for the benchmark functions and emulators.
    The numerical suffix in the legend indicates the number of MC samples.
    The top and bottom rows show the results on the benchmark functions and the emulators, respectively.
    }
    \label{fig:exp_benchmark_emulator}
\end{figure*}

\section{Experiments}
\label{sec:experiments}

We conducted numerical experiments using synthetic functions generated from GPs, benchmark functions shown in \url{https://www.sfu.ca/~ssurjano/optimization.html}, and the emulators proposed by \citet{hase2021olympus}.
We compare OVR and ROVR with EI \citep{Mockus1978-Application}, UCB \citep{Srinivas2010-Gaussian}, TS \citep{Russo2014-learning}, PI from the maximum of a sample path (PIMS) \citep{takeno2024-posterior}, JES \citep{hvarfner2022-joint}, MES \citep{Wang2017-Max}, and KG~\citep{Frazier2009-knowledge}.
Note that we do not employ the $\epsilon$-greedy modification of JES \citep{hvarfner2022-joint} to focus on the differences between AFs.
For TS, PIMS, JES, MES, OVR, and ROVR, we employ the posterior sampling approximation in \citep{Wilson2020-efficiently,wilson2021pathwise}.
We used GPy \citep{GPy2014} and SciPy \citep{SciPy2020} for implementations.
For sensitivity to the number of MC samples, we show the synthetic function experiment results for varying numbers of MC samples in Appendix~\ref{sec:exp_sensitivity_MC}.

\subsection{Synthetic Function Experiments}

We used synthetic functions generated from GPs that match the assumptions of our analysis.
We set $\cX = \{0.0, 0.1, \dots, 0.9 \}^d$ and $d = 4$.
We employed the SE and Mat\'ern kernels with $\nu = 5/2$ and $\ell = 0.3$.
Furthermore, we actually added observation noise $\varepsilon \sim \cN(0, \sigma^2)$ with $\sigma^2 \in \{10^{-4}, 1 \}$.
We fixed the GP model's hyperparameters to match the parameters used to generate the objective functions and observation noise.
We constructed the initial dataset from the grid points closest to $5$ points generated using a randomized Sobol sequence.
We set $\beta_t$ as the theoretical value $\beta_t = 2\log(|\cX|t^2 / \sqrt{2 \pi} + 1)$ for UCB \citep{Takeno2023-randomized} and $c_t = 0.1 \ln^{-d} (1 + t)$ for ROVR.
Note that the vanishing BSR property still holds for this choice of $c_t$ if we combine it with the sufficiently increasing $M_t$.
Finally, we measured optimization performance by the simple regret $f(\*x^*) - f(\hat{\*x}_t)$, where $\hat{\*x}_t = \argmax_{\*x \in \cX} \mu_t(\*x)$.

Figure~\ref{fig:exp_synthetic} shows the average and standard error of simple regret over $16$ random trials.
In these experiments, OVR and ROVR show similar results.
When $\sigma^2 = 10^{-4}$, OVR and ROVR consistently outperform MES, JES, and UCB and are comparable to EI, TS, KG, and PIMS.
When $\sigma^2 = 1$, OVR and ROVR show comparable or slightly worse performance compared with KG.
For other baseline methods, OVR and ROVR show superior or comparable performance for the SE kernels and Mat\'ern kernels, respectively.
Consequently, OVR and ROVR show comparable performance to KG and often outperform ES-based methods.

\subsection{Benchmark and Emulator Experiments}

We used the benchmark functions (Ackley, Hartmann3, Hartmann4, and Hartmann6) and the emulators \citep{hase2021olympus} (colors\_bob, hplc, snar, and suzuki).
Additional experiments for other benchmark functions and emulators are shown in Appendix~\ref{sec:exp_others}.
We employed the SE kernel.
Furthermore, we actually added observation noise $\varepsilon \sim \cN(0, \sigma^2)$ with small noise variance $\sigma^2 = 10^{-8}$.
We optimized the hyperparameters of the GP model every $5$ iterations.
We constructed an initial dataset of $5d$ data points generated using a randomized Sobol sequence.
We set $\beta_t$ as the heuristic value $\beta_t = 0.2 \ln (2 (t+1))$ \citep{kandasamy2015-high} for UCB and set $c_t = 0.1 \ln^{-d} (1 + t)$ for ROVR.
Finally, we measured optimization performance by $f(\*x^*) - \max\{f(\*x_1), \dots, f(\*x_t), f(\hat{\*x}_t) \}$ for the benchmark functions and $\max\{f(\*x_1), \dots, f(\*x_t), f(\hat{\*x}_t) \}$ for the emulators, where $\hat{\*x}_t = \argmax_{\*x \in \cX} \{\mu_t(\*x) - 2 \sigma_t(\*x)\}$, as in \citep{Takeno2020-Multifidelity,Takeno2022-generalized} since the noise is quite small and $\argmax_{\*x \in \cX} \mu_t(\*x)$ is often unstable due to hyperparameter estimation.

Figure~\ref{fig:exp_benchmark_emulator} shows the average and standard error of the evaluation metrics over $16$ random trials.
In most cases, OVR and ROVR yield similar results, indicating that the regularization in ROVR does not degrade practical performance.
For most results, OVR and ROVR show performance in the top group.
In particular, OVR and ROVR often outperform one-step lookahead methods, such as MES, JES, and KG.
Hence, particularly compared with one-step lookahead methods, we can observe the superiority of OVR and ROVR in both the benchmark functions and emulator experiments.

\section{Conclusion}
\label{sec:conclusion}

%
We establish uniform MC estimation error bounds for OVR and ROVR and a vanishing BSR upper bound for ROVR.
Finally, we confirm the performance of OVR and ROVR via numerical experiments.

\paragraph{Limitations and future work}
As discussed in Section~\ref{sec:analysis_regret}, the cumulative regret analysis remains an open problem.
Furthermore, determining whether we can tighten the BSR upper bound and whether we can remove the regularization is important.
In addition, we leave the analysis incorporating the approximation error for posterior sampling to future work.
On the other hand, we expect that OVR and ROVR are extendable to more complex optimization problems, as in other one-step lookahead BO methods, such as multi-fidelity BO \citep{Takeno2020-Multifidelity,Takeno2022-generalized,Moss2021-gibbon}, constrained BO \citep{Hernandez-Lobato2015-Predictive,takeno2022-sequential}, and multi-objective BO \citep{Poloczek2017-Multi,Suzuki2020-multi,inatsu2024-bounding}.
Extensions to these settings that maintain the theoretical guarantee seem to be a promising direction.

\section*{Acknowledgements}
This work was supported by JST PRESTO Grant Number JPMJPR24J6 and JSPS KAKENHI Grant Number JP24K20847.

\bibliography{ref}
\bibliographystyle{plainnat}
\clearpage
\appendix
\thispagestyle{empty}

\section{Auxiliary Results}
\label{sec:lemmas}

We list auxiliary results here.

\begin{theorem}[Hoeffding's inequality, Theorem 2.2.6 in \citep{Vershynin2018-high}]
    \label{thm:hoeffding}
    Let $X_1, \dots, X_N$ be independent random variables.
    Assume that $X_i \in [m_i, M_i]$ for all $i \in [N]$.
    Then, for any $c > 0$, we have
    \begin{align*}
        \Pr \rbr{\frac{1}{N}\sum_{i = 1}^N \rbr{ X_i - \EE[X_i]} \geq c }
        \leq \exp \rbr{- \frac{2N^2 c^2}{\sum_{i = 1}^N (M_i - m_i)^2}}
    \end{align*}
\end{theorem}



\begin{lemma}[Lemma H.1 of \citep{Takeno2023-randomized}, Lemma E.1 of \citep{takeno2025-regret}]
    \label{lem:L_f_bound}
    Let $f \sim \cG \cP (0, k)$ and suppose that Assumption~\ref{assump:Bayesian_continuous} holds.
    Then, the following inequality holds:
    \begin{align}
        \EE \sbr{\max_{j \in [d]} \sup_{\*x \in \cX} \left| \frac{\partial f (\*u)}{\partial u_j} \biggm|_{\*u = \*x} \right|}
        \leq b \bigl( \sqrt{\log (ad)} + \sqrt{\pi} / 2 \bigr).
    \end{align}
\end{lemma}

\begin{lemma}[Exercise 4.2.10 in \citep{Vershynin2018-high}]
    \label{lem:CN_monotonicity}
    Let $\cN(\cX, d, \epsilon)$ with some metric $d$ and $\epsilon > 0$ be the covering number of a domain $\cX$.
    Then, if $\widetilde{\cX} \subset \cX$, the following inequality holds:
    \begin{align*}
        \cN(\widetilde{\cX}, d, \epsilon) \leq \cN(\cX, d, \epsilon / 2).
    \end{align*}
\end{lemma}

\section{Proof for Section~\ref{sec:analysis_MC}}
\label{sec:proof_MC}

\subsection{Proof for Finite Domains}
\label{sec:proof_MC_discrete}

\thmMCErrorDiscrete*
\begin{proof}
    From Theorem~\ref{thm:hoeffding} and the union bound, it is trivial to obtain
    \begin{align*}
        \forall M_t \in \NN, \forall \*x \in \cX,
        \Pr\rbr{|\alpha_t (\*x) - \hat{\alpha}_t (\*x)| \leq \sqrt{\frac{\log(2 / \delta)}{2M_t}} \biggm| \cD_{t-1}}
        \geq 1 - \delta,
    \end{align*}
    for any $\cD_{t-1}$ and $\delta \in (0, 1)$, since $\sigma_t(\*x^* \mid \*x) \in [0, 1]$ from Assumption~\ref{assump:Bayesian} and the monotone decreasing property of the posterior variance.
    Thus, taking the union bound for all $\*x \in \cX$ with failure probability $\delta / |\cX|$ at each point, we can obtain the desired result.
\end{proof}

\subsection{Proof for Continuous Domains}
\label{sec:proof_MC_continuous}

First, we show the following supporting lemma.
\lemLipschitzContinuityPosteriorConditionalSTD*
\begin{proof}
    Fix any $t \in \NN$ and $\cD_{t-1}$.
    Let $\bar{f}(\cdot) = f(\cdot) - \mu_{t-1}(\cdot)$, so that $\bar{f} \mid \cD_{t-1} \sim \cG \cP(0, k_{t-1})$, and let $\xi \sim \cN(0, \sigma^2)$ be independent of $\bar{f}$ conditional on $\cD_{t-1}$.
    We use the same auxiliary noise $\xi$ for all candidate locations.
    Then, we can rewrite the posterior variance as follows:
    \begin{align}
        \sigma_t^2 (\*x \mid \*u)
        &= \sigma_{t-1}^2(\*x) - \frac{k_{t-1}(\*x, \*u)^2}{\sigma_{t-1}^2(\*u) + \sigma^2} \\
        &= \min_{a \in \RR} \EE \sbr{\rbr{ \bar{f}(\*x) - a \rbr{\bar{f}(\*u) + \xi} }^2 \bigm| \cD_{t-1}}.
    \end{align}
    This equality can be confirmed through square completion as follows:
    \begin{align}
        \EE \sbr{\rbr{ \bar{f}(\*x) - a \rbr{\bar{f}(\*u) + \xi} }^2 \bigm| \cD_{t-1}}
        &= \sigma_{t-1}^2(\*x) + a^2 \rbr{\sigma_{t-1}^2(\*u) + \sigma^2} - 2 a k_{t-1}(\*x, \*u) \\
        &= \rbr{\sigma_{t-1}^2(\*u) + \sigma^2} \rbr{a - \frac{k_{t-1}(\*x, \*u)}{\sigma_{t-1}^2(\*u) + \sigma^2}}^2 + \sigma_{t-1}^2(\*x) - \frac{k_{t-1}(\*x, \*u)^2}{\sigma_{t-1}^2(\*u) + \sigma^2}.
    \end{align}
    Therefore, by defining $a_{\*v} = \frac{k_{t-1}(\*x, \*v)}{\sigma_{t-1}^2(\*v) + \sigma^2}$, we have
    \begin{align}
        \sigma_t (\*x \mid \*u)
        &\leq \sqrt{\EE \sbr{\rbr{ \bar{f}(\*x) - a_{\*v} \rbr{\bar{f}(\*u) + \xi} }^2 \bigm| \cD_{t-1}}} \\
        &\leq \sqrt{\EE \sbr{\rbr{ \bar{f}(\*x) - a_{\*v} \rbr{\bar{f}(\*v) + \xi} }^2 \bigm| \cD_{t-1}}} \\
        &\quad + |a_{\*v}| \sqrt{\EE\sbr{ \rbr{\bar{f}(\*v) - \bar{f}(\*u)}^2 \bigm| \cD_{t-1}}} \\
        &= \sigma_t (\*x \mid \*v) + |a_{\*v}| \rho_{t-1}(\*u, \*v),
    \end{align}
    where we define $\rho_{t-1}(\*u, \*v)^2 = k_{t-1}(\*u, \*u) + k_{t-1}(\*v, \*v) - 2 k_{t-1}(\*u, \*v)$.
    The second inequality is the triangle inequality for the conditional $L^2$ norm.
    The last equality follows because $a_{\*v}$ is the minimizing coefficient at $\*v$ and $\bar{f} \mid \cD_{t-1} \sim \cG \cP(0, k_{t-1})$.
    By symmetry, we can obtain the following inequality:
    \begin{align}
        |\sigma_t (\*x \mid \*u) - \sigma_t (\*x \mid \*v)|
        \leq \max\cbr{|a_{\*v}|, |a_{\*u}|} \rho_{t-1}(\*u, \*v).
        \label{eq:posterior_variance_Lipschitz_mid}
    \end{align}

    If $\sigma_{t-1}(\*v) = 0$, then $k_{t-1}(\*x, \*v) = 0$ by the Cauchy--Schwarz inequality, so $a_{\*v} = 0$.
    Otherwise, by the Cauchy--Schwarz inequality and $\sigma_{t-1}^2(\*v) + \sigma^2 \geq 2 \sigma \sigma_{t-1}(\*v)$, we have
    \begin{align}
        |a_{\*v}|
        &= \frac{|k_{t-1}(\*x, \*v)|}{\sigma_{t-1}^2(\*v) + \sigma^2} \\
        &\leq \frac{\sigma_{t-1}(\*x) \sigma_{t-1}(\*v)}{\sigma_{t-1}^2(\*v) + \sigma^2} \\
        &\leq \frac{\sigma_{t-1}(\*x) \sigma_{t-1}(\*v)}{2 \sigma \sigma_{t-1}(\*v)} \\
        &\leq \frac{1}{2 \sigma},
        \label{eq:a_v_bound}
    \end{align}
    where we use the fact that $\sigma_{t-1}(\*x) \leq 1$ for all $\*x \in \cX$ from Assumption~\ref{assump:Bayesian} and the monotone decreasing property of the posterior variance.
    In addition, from the definition of $\rho_{t-1}(\*u, \*v)$, we have
    \begin{align}
        \rho_{t-1}(\*u, \*v)^2
        &= \rho(\*u, \*v)^2 - \rbr{\*k_{t-1}(\*u) - \*k_{t-1}(\*v)}^\top \rbr{\*K_{t-1} + \sigma^2 \*I_{t-1}}^{-1} \rbr{\*k_{t-1}(\*u) - \*k_{t-1}(\*v)}
        \leq \rho(\*u, \*v)^2.
    \end{align}
    Therefore, from Assumption~\ref{assump:MCEstimation_continuous}, we have
    \begin{align}
        \rho_{t-1}(\*u, \*v)
        \leq \rho(\*u, \*v)
        &\leq L_k \| \*u - \*v \|_1.
        \label{eq:rho_Lipschitz}
    \end{align}
    Consequently, substituting Eqs.~\eqref{eq:a_v_bound} and \eqref{eq:rho_Lipschitz} into Eq.~\eqref{eq:posterior_variance_Lipschitz_mid}, we have
    \begin{align}
        |\sigma_t (\*x \mid \*u) - \sigma_t (\*x \mid \*v)|
        &\leq \frac{L_k}{2 \sigma} \| \*u - \*v \|_1,
        \label{eq:posterior_variance_Lipschitz}
    \end{align}
    which completes the proof.
\end{proof}

\begin{theorem}[MC estimation error bound for continuous domains]
    \label{thm:MCError_continuous_formal}
    Pick $\delta \in (0, 1)$.
    Suppose that Assumption~\ref{assump:MCEstimation_continuous} holds.
    Then, for any $\cD_{t-1}$, $M_t \in \NN$, and $c_t \geq 0$, we have
    \begin{align}
        \Pr\rbr{\max_{\*x \in \cX} |\alpha_t(\*x) - \hat{\alpha}_t(\*x)| \leq \sqrt{\frac{B_t - \log \delta}{2M_t}} + \frac{1}{\sqrt{M_t}} \biggm| \cD_{t-1} }
        \geq 1 - \delta,
    \end{align}
    and
    \begin{align}
        \EE \sbr{\max_{\*x \in \cX} |\alpha_t(\*x) - \hat{\alpha}_t(\*x)| \biggm| \cD_{t-1}}
        \leq
        \sqrt{\frac{B_t + 1}{2M_t}} + \frac{1}{\sqrt{M_t}},
    \end{align}
    where $L_{\alpha} = L_k / \sigma$ and $B_t = \log 2 + d \log\rbr{1 + dr \sqrt{M_t} L_{\alpha}}$.
\end{theorem}
\begin{proof}
    From the assumption on $\cX$, there exists an $(1 / (\sqrt{M_t} L_{\alpha}))$-net $\cX_t \subset \cX$ that satisfies
    $\max_{\*x \in \cX} \min_{\*x^\prime \in \cX_t} \| \*x - \*x^\prime \|_1 \leq 1 / (\sqrt{M_t} L_{\alpha})$
    and $|\cX_t| \leq \rbr{1 + dr \sqrt{M_t} L_{\alpha}}^d$.
    To construct this net, partition each coordinate interval $[0, r]$ into $\lceil dr \sqrt{M_t} L_{\alpha} \rceil$ equal subintervals and choose one representative from each resulting cell that intersects $\cX$, with the representative lying in $\cX$.
    Each cell has L1 diameter at most $1 / (\sqrt{M_t} L_{\alpha})$, and the number of cells is at most $\rbr{1 + dr \sqrt{M_t} L_{\alpha}}^d$.

    Thus, by the same proof as in Theorem~\ref{thm:MCError_discrete}, we can obtain
    \begin{align}
        \forall M_t \in \NN,
        \Pr\rbr{\max_{\*x \in \cX_t} |\alpha_t(\*x) - \hat{\alpha}_t(\*x)| \leq \sqrt{\frac{\log(2 |\cX_t| / \delta)}{2M_t}} \biggm| \cD_{t-1} }
        \geq 1 - \delta.
    \end{align}
    Furthermore, from Lemma~\ref{lemLipschitzContinuityPosteriorConditionalSTD}, $E_t(\*x) = \alpha_t(\*x) - \hat{\alpha}_t(\*x)$ is Lipschitz continuous with respect to $\*x$ with Lipschitz constant $L_{\alpha} = L_k / \sigma$.
    Since the regularization terms cancel, we have
    \begin{align}
        E_t(\*u) - E_t(\*v)
        &= \EE\sbr{\sigma_t(\*x^* \mid \*u) - \sigma_t(\*x^* \mid \*v) \bigm| \cD_{t-1}}  - \frac{1}{M_t} \sum_{m=1}^{M_t} \cbr{\sigma_t(\*x_{t,m}^* \mid \*u) - \sigma_t(\*x_{t,m}^* \mid \*v)}.
    \end{align}
    Thus, by the triangle inequality,
    \begin{align}
        |E_t(\*u) - E_t(\*v)|
        &\leq \EE\sbr{|\sigma_t(\*x^* \mid \*u) - \sigma_t(\*x^* \mid \*v)| \bigm| \cD_{t-1}} + \frac{1}{M_t} \sum_{m=1}^{M_t} |\sigma_t(\*x_{t,m}^* \mid \*u) - \sigma_t(\*x_{t,m}^* \mid \*v)| \\
        &\leq \frac{L_k}{2 \sigma} \| \*u - \*v \|_1 + \frac{1}{M_t} \sum_{m=1}^{M_t} \frac{L_k}{2 \sigma} \| \*u - \*v \|_1
        = \frac{L_k}{\sigma} \| \*u - \*v \|_1.
    \end{align}
    Let $[\*x]_t \in \argmin_{\*x^\prime \in \cX_t} \|\*x - \*x^\prime\|_1$ be a nearest net point.
    From the construction of $\cX_t$, for every $\*x \in \cX$, we have
    \begin{align}
        |E_t(\*x)|
        &\leq |E_t([\*x]_t)| + |E_t(\*x) - E_t([\*x]_t)| \\
        &\leq \max_{\*x^\prime \in \cX_t} |E_t(\*x^\prime)| + L_{\alpha} \|\*x - [\*x]_t\|_1 \\
        &\leq \max_{\*x^\prime \in \cX_t} |E_t(\*x^\prime)| + \frac{1}{\sqrt{M_t}},
    \end{align}
    where $L_{\alpha} = L_k / \sigma$.
    Hence, for all $M_t \in \NN$, we have
    \begin{align}
        \Pr\rbr{\max_{\*x \in \cX} |\alpha_t(\*x) - \hat{\alpha}_t(\*x)| \leq \sqrt{\frac{B_t - \log \delta}{2M_t}} + \frac{1}{\sqrt{M_t}} \biggm| \cD_{t-1} }
        \geq 1 - \delta,
    \end{align}
    which is the desired result for the probability bound.

    Next, we show the upper bound with respect to the expectation.
    From the probability lower bound above, for any $\delta \in (0, 1)$, $\cD_{t-1}$, $c_t$, and $M_t$, we obtain
    \begin{align}
        F_t\rbr{ \sqrt{\frac{B_t - \log \delta}{2M_t}} + \frac{1}{\sqrt{M_t}} }
        \geq 1 - \delta,
    \end{align}
    where $F_t$ is the cumulative distribution function of $\max_{\*x \in \cX} |\alpha_t(\*x) - \hat{\alpha}_t(\*x)| \bigm| \cD_{t-1}$.
    Thus, by substituting $U \sim {\rm Uni}((0, 1))$ and taking the expectation with respect to $U$, we have
    \begin{align}
        &F_t\rbr{ \sqrt{\frac{B_t - \log U}{2M_t}} + \frac{1}{\sqrt{M_t}} }
        \geq 1 - U \\
        &\Leftrightarrow
        \sqrt{\frac{B_t - \log U}{2M_t}} + \frac{1}{\sqrt{M_t}}
        \geq F_t^{-1}(1 - U) \\
        &\Rightarrow
        \EE_U \sbr{\sqrt{\frac{B_t - \log U}{2M_t}} + \frac{1}{\sqrt{M_t}}}
        \geq \EE_U \sbr{ F_t^{-1}(1 - U)} = \EE_U \sbr{ F_t^{-1}(U)},
    \end{align}
    where $F_t^{-1}$ is a generalized inverse function of $F_t$.
    Note that $1 - U \sim {\rm Uni}((0, 1))$.
    Since $F_t^{-1}(U) \sim p\rbr{\max_{\*x \in \cX} |\alpha_t(\*x) - \hat{\alpha}_t(\*x)| \bigm| \cD_{t-1}}$, we see that
    \begin{align}
        \EE_U \sbr{ F_t^{-1}(U)}
        = \EE\sbr{\max_{\*x \in \cX} |\alpha_t(\*x) - \hat{\alpha}_t(\*x)| \bigm| \cD_{t-1}}.
    \end{align}
    Consequently, we have
    \begin{align}
        \EE\sbr{\max_{\*x \in \cX} |\alpha_t(\*x) - \hat{\alpha}_t(\*x)| \bigm| \cD_{t-1}}
        &\leq
        \EE_U \sbr{\sqrt{\frac{B_t - \log U}{2M_t}} + \frac{1}{\sqrt{M_t}}} \\
        &\leq
        \sqrt{ \frac{B_t - \EE_U\sbr{\log U}}{2M_t}} + \frac{1}{\sqrt{M_t}} \\
        &\leq
        \sqrt{ \frac{B_t + 1}{2M_t}} + \frac{1}{\sqrt{M_t}},
    \end{align}
    where the second and third inequalities follow from Jensen's inequality and $-\log U \sim {\rm Exp}(1)$, respectively.
\end{proof}

\section{Proof for Section~\ref{sec:analysis_regret}}
\label{sec:proof_regret}

First, we show the Lipschitz continuity of the posterior STD.
\LemLipschitzPosteriorSTD*
\begin{proof}
    Fix any $t \in \NN$ and $\cD_{t-1}$, and define
    $\rho_{t-1}(\*u, \*v)^2 = k_{t-1}(\*u, \*u) + k_{t-1}(\*v, \*v) - 2 k_{t-1}(\*u, \*v)$.
    By the Cauchy--Schwarz inequality for the posterior covariance, we have
    $k_{t-1}(\*u, \*v) \leq \sigma_{t-1}(\*u) \sigma_{t-1}(\*v)$, and hence
    \begin{align*}
        \rbr{\sigma_{t-1}(\*u) - \sigma_{t-1}(\*v)}^2
        &\leq k_{t-1}(\*u, \*u) + k_{t-1}(\*v, \*v) - 2 k_{t-1}(\*u, \*v)
        = \rho_{t-1}(\*u, \*v)^2.
    \end{align*}
    For $t \geq 2$, the posterior covariance formula and the positive definiteness of $\*K_{t-1} + \sigma^2 \*I_{t-1}$ imply
    \begin{align*}
        \rho_{t-1}(\*u, \*v)^2
        = \rho(\*u, \*v)^2 - \rbr{\*k_{t-1}(\*u) - \*k_{t-1}(\*v)}^\top \rbr{\*K_{t-1} + \sigma^2 \*I_{t-1}}^{-1} \rbr{\*k_{t-1}(\*u) - \*k_{t-1}(\*v)}
        \leq \rho(\*u, \*v)^2.
    \end{align*}
    The same inequality holds for $t = 1$.
    Taking square roots and applying Assumption~\ref{assump:MCEstimation_continuous} gives
    \begin{align*}
        |\sigma_{t-1}(\*u) - \sigma_{t-1}(\*v)|
        \leq \rho_{t-1}(\*u, \*v)
        \leq \rho(\*u, \*v)
        \leq L_k \| \*u - \*v \|_1,
    \end{align*}
    which completes the proof.
\end{proof}

\thmBSRBound*
\begin{proof}
    First, we show the proof for the case of finite domains.
    We can see that
    \begin{align}
        {\rm BSR}_T
        &= \EE \sbr{f(\*x^*) - f(\hat{\*x}_T)} \\
        &= \EE \sbr{f(\*x^*) - \mu_{T}(\*x^*) + \mu_{T}(\*x^*) - \mu_{T}(\hat{\*x}_T)} \\
        &\leq \EE \sbr{f(\*x^*) - \mu_{T}(\*x^*)},
    \end{align}
    where we use the facts that
    $\EE[f(\hat{\*x}_T)] = \EE[ \EE[f(\hat{\*x}_T) \mid \cD_{T}]] = \EE[\mu_{T}(\hat{\*x}_T)]$
    and
    $\hat{\*x}_T = \argmax_{\*x \in \cX} \mu_{T}(\*x)$.
    Then, we decompose the regret as in the prior works \citep{Russo2014-learning,Takeno2023-randomized,Kandasamy2018-Parallelised,paria2020-flexible}:
    \begin{align}
        {\rm BSR}_T
        &\leq \EE \sbr{f(\*x^*) - U_T(\*x^*)} + \EE \sbr{U_T(\*x^*) - \mu_{T}(\*x^*)}, \\
        &= \underbrace{\EE \sbr{f(\*x^*) - U_T(\*x^*)}}_{R_1} + \beta_T^{1/2} \underbrace{\EE \sbr{\sigma_{T} (\*x^*) }}_{R_2},
    \end{align}
    where $U_T (\*x) = \mu_T(\*x) + \beta_T^{1/2} \sigma_T(\*x)$.

    For $R_1$, we apply the same proof as in \citep{Russo2014-learning,Takeno2023-randomized,Kandasamy2018-Parallelised,paria2020-flexible}.
    We use the following fact for a Gaussian random variable $Z \sim \cN(m, s^2)$ with $m \leq 0$:
    \begin{align}
        \EE[Z_+] \leq \frac{s}{\sqrt{2 \pi}} \exp \left\{ - \frac{m^2}{2 s^2} \right\}, \label{eq:Gauss_plus_expectation_1}
    \end{align}
    where $Z_+ \coloneqq \max \{0, Z\}$.
    Then, we can arrange $R_1$ as follows:
    \begin{align}
        R_1
        &= \EE_{\cD_{T}} [ \EE[ f(\*x^*) - U_T(\*x^*) \mid \cD_{T}]] \\
        &\leq \EE_{\cD_{T}} [ \EE[ \bigl( f(\*x^*) - U_T(\*x^*) \bigr)_+  \mid \cD_{T}]] \\
        &\leq \sum_{\*x \in \cX} \EE_{\cD_{T}} \left[ \EE \left[ \bigl(f(\*x) - U_T(\*x)\bigr)_+ \bigm| \cD_{T} \right] \right],
    \end{align}
    where the last inequality follows from $\bigl( f(\*x^*) - U_T(\*x^*) \bigr)_+ \leq \sum_{\*x \in \cX} \bigl( f(\*x) - U_T(\*x) \bigr)_+$.
    Then, since $f(\*x) - U_T(\*x) \mid \cD_{T} \sim \cN( -\beta_T^{1/2}\sigma_{T}(\*x), \sigma^2_{T}(\*x))$, we can apply Eq.~\eqref{eq:Gauss_plus_expectation_1}:
    \begin{align}
        R_1
        &\leq \sum_{\*x \in \cX} \EE_{\cD_{T}} \left[ \frac{\sigma_{T}(\*x)}{\sqrt{2 \pi}} \exp \left\{ - \frac{\beta_T}{2} \right\} \right] \\
        &\leq |\cX| \frac{1}{\sqrt{2 \pi}} \exp \left\{ - \frac{\beta_T}{2} \right\} \\
        &\leq \frac{1}{T},
    \end{align}
    where the second inequality follows from $\sigma_T(\*x) \leq 1$ due to $k(\*x, \*x) \leq 1$ for all $\*x \in \cX$.

    For $R_2$, from the definition of AF, we see that, for any $t \in \NN$, $\cD_{t-1}$, and $c_t > 0$,
    \begin{align*}
        \frac{1}{M_t} \sum_{m=1}^{M_t} \sigma_{t}(\*x^*_{t, m} \mid \*x_t) - c_t \sigma_{t-1}(\*x_t)
        &\leq
        \frac{1}{M_t} \sum_{m=1}^{M_t} \sigma_{t}(\*x^*_{t, m} \mid \*x^*) - c_t \sigma_{t-1}(\*x^*) \\
        &\leq
        \frac{1}{M_t} \sum_{m=1}^{M_t} \sigma_{t-1}(\*x^*_{t, m}) - c_t \sigma_{t-1}(\*x^*),
    \end{align*}
    where we use the fact that $\sigma_{t-1}(\*x) \geq \sigma_t(\*x)$ for all $\*x \in \cX$.
    Let $\overline{E_t} = \max_{\*x \in \cX} |\alpha_t(\*x) - \hat{\alpha}_t(\*x)|$.
    The preceding inequality holds for every realization of the MC samples and $\*x^*$, and therefore
    \begin{align*}
        \alpha_t(\*x_t)
        &\leq \hat{\alpha}_t(\*x_t) + \overline{E_t} \\
        &\leq \frac{1}{M_t} \sum_{m=1}^{M_t} \sigma_{t-1}(\*x^*_{t,m})
        - c_t \sigma_{t-1}(\*x^*) + \overline{E_t}.
    \end{align*}
    Conditional on $\cD_{t-1}$, the MC samples are independent draws from $p(\*x^* \mid \cD_{t-1})$ and are independent of the true function $f$.
    Hence, $\*x_t$, which is chosen using these MC samples, is conditionally independent of the true maximizer $\*x^*$.
    Taking the expectation of the entire inequality over both the MC samples and $\*x^*$, conditional on $\cD_{t-1}$, gives
    \begin{align*}
        \EE\sbr{\sigma_t(\*x^*) \bigm| \cD_{t-1}}
        - c_t \EE\sbr{\sigma_{t-1}(\*x_t) \bigm| \cD_{t-1}}
        &\leq (1-c_t) \EE\sbr{\sigma_{t-1}(\*x^*) \bigm| \cD_{t-1}}
        + \EE\sbr{\overline{E_t} \bigm| \cD_{t-1}},
    \end{align*}
    where we use $\sigma_t(\*x^*) = \sigma_t(\*x^* \mid \*x_t)$ and the definition of $\alpha_t$.
    Integrating the probability bound in Theorem~\ref{thm:MCError_discrete}, using the same quantile argument as in Theorem~\ref{thm:MCError_continuous_formal}, yields
    $\EE\sbr{\overline{E_t} \mid \cD_{t-1}} \leq \sqrt{\rbr{\log(2|\cX|)+1}/(2M_t)}$ for finite domains.
    For continuous domains, Theorem~\ref{thm:MCError_continuous_formal} applies directly.
    Thus, in either case,
    \begin{align*}
        \EE\sbr{\overline{E_t} \bigm| \cD_{t-1}}
        = O\rbr{\sqrt{\frac{\log(M_t)}{M_t}}},
    \end{align*}
    with a constant independent of $t$, $M_t$, and $\cD_{t-1}$.
    Rearranging the conditional inequality and then taking the expectation over $\cD_{t-1}$, we obtain
    \begin{align}
        \EE\sbr{\sigma_{t-1}(\*x^*)}
        &\leq
        \EE\sbr{\sigma_{t-1}(\*x_t) }
        + \frac{1}{c_t}\rbr{\EE \sbr{ \sigma_{t-1}(\*x^*)} - \EE \sbr{ \sigma_{t}(\*x^*)}}
        + O\rbr{\frac{\sqrt{\log(M_t)}}{\sqrt{M_t} c_t}}.
    \end{align}
    Moreover, from the monotone decreasing property of $\sigma_t(\cdot)$ with respect to $t$, we see that
    \begin{align*}
        \EE [\sigma_1(\*x^*)]
        &\geq \dots
        \geq
        \EE [\sigma_t(\*x^*)]
        \dots
        \geq
        \EE [\sigma_T(\*x^*)].
    \end{align*}
    Hence, we can obtain the upper bound of $R_2$ as follows:
    \begin{align}
        R_2
        &\leq \frac{1}{T} \sum_{t=1}^T \EE [\sigma_{t-1}(\*x^*)] \\
        &\leq
        \frac{1}{T} \sum_{t=1}^T \cbr{ \EE [\sigma_{t-1}(\*x_t)]
        + \frac{1}{c_t}\rbr{\EE \sbr{ \sigma_{t-1}(\*x^*)} - \EE \sbr{ \sigma_{t}(\*x^*)}}
        + O\rbr{\frac{\sqrt{\log(M_t)}}{\sqrt{M_t} c_t}} } \\
        &=
        \frac{1}{T} \EE \sbr{\sum_{t=1}^T \sigma_{t-1}(\*x_t)}
        + \frac{1}{T} \sum_{t=1}^T \frac{1}{c_t}\rbr{\EE \sbr{ \sigma_{t-1}(\*x^*)} - \EE \sbr{ \sigma_{t}(\*x^*)}}
        + O\rbr{\frac{1}{T} \sum_{t=1}^T \frac{\sqrt{\log(M_t)}}{\sqrt{M_t} c_t}}.
    \end{align}

    Then, the first term can be bounded from above as follows:
    \begin{align}
        \frac{1}{T} \EE \sbr{\sum_{t=1}^T \sigma_{t-1}(\*x_t)}
        &\leq \frac{1}{T} \EE \sbr{\sqrt{T \sum_{t=1}^T \sigma_{t-1}^2(\*x_t)}} \\
        &\leq  \sqrt{ \frac{C_1 \gamma_T}{T}},
    \end{align}
    where the inequalities follow from the Cauchy--Schwarz inequality and Lemma~5.4 in \citep{Srinivas2010-Gaussian}.
    For the second term, we can see that
    \begin{align}
        \sum_{t=1}^T \frac{1}{c_t}\rbr{\EE \sbr{ \sigma_{t-1}(\*x^*)} - \EE \sbr{ \sigma_{t}(\*x^*)}}
        &=
        \frac{\EE[\sigma_{0}(\*x^*)]}{c_1}
        + \sum_{t=2}^T \rbr{\frac{1}{c_t} - \frac{1}{c_{t-1}}} \EE[\sigma_{t-1}(\*x^*)]
        - \frac{\EE[\sigma_{T}(\*x^*)]}{c_T}
        \\
        &\leq
        \frac{\EE[\sigma_{0}(\*x^*)]}{c_1}
        + \EE[\sigma_{0}(\*x^*)] \sum_{t=2}^T \rbr{\frac{1}{c_t} - \frac{1}{c_{t-1}}}
        - \frac{\EE[\sigma_{T}(\*x^*)]}{c_T}
        \\
        &=
        \frac{\EE[\sigma_{0}(\*x^*)]}{c_1}
        + \EE[\sigma_{0}(\*x^*)] \rbr{\frac{1}{c_T} - \frac{1}{c_1}}
        - \frac{\EE[\sigma_{T}(\*x^*)]}{c_T}
        \\
        &= \frac{\EE[\sigma_{0}(\*x^*)] - \EE[\sigma_{T}(\*x^*)]}{c_T} \\
        &\leq \frac{1}{c_T},
    \end{align}
    where the inequality follows from $\EE [\sigma_1(\*x^*)]  \geq \dots \geq \EE [\sigma_t(\*x^*)] \dots \geq \EE [\sigma_T(\*x^*)]$.
    Therefore, we obtained the desired result for discrete input domains.


    Second, we show the proof for the case where Assumption~\ref{assump:Bayesian_continuous} holds.
    From Lemma~\ref{lem:L_f_bound}, we obtain
    \begin{align*}
        \EE \sbr{\max_{j \in [d]} \sup_{\*x \in \cX} \left| \frac{\partial f (\*u)}{\partial u_j} \biggm|_{\*u = \*x} \right|}
        \leq b \bigl( \sqrt{\log (ad)} + \sqrt{\pi} / 2 \bigr)
        \eqqcolon L_f.
    \end{align*}
    Then, let $L = \max \cbr{L_f, L_k}$, as in the theorem.
    Furthermore, let $\cX_T \subset \cX$ be a finite subset and $[\*x]_T \in \argmin_{\*x^\prime \in \cX_T} \|\*x - \*x^\prime \|_1$.
    From the assumption on $\cX$, we can choose $\cX_T$ as a $(1 / (T L))$-net in $\cX$ that satisfies
    $\max_{\*x \in \cX} \| \*x - [\*x]_T \|_1 \leq 1 / (T L)$
    and $|\cX_T| \leq \rbr{1 + dr T L }^d$.
    To construct this net, partition each coordinate interval $[0, r]$ into $\lceil drTL \rceil$ equal subintervals and choose one representative in $\cX$ from each cell that intersects $\cX$.
    Each cell has L1 diameter at most $1/(TL)$, and there are at most $(1 + drTL)^d$ cells.

    Then, as with the case of finite domains, we can derive
    \begin{align}
        {\rm BSR}_T
        &= \EE \sbr{f(\*x^*) - f(\hat{\*x}_T)} \\
        &\leq \EE \sbr{f(\*x^*) - \mu_{T}([\*x^*]_T)}.
    \end{align}
    Moreover, with $\beta_T = 2\log\rbr{1 + (1 + drTL)^d T}$, we have
    $|\cX_T| \exp(-\beta_T/2)/\sqrt{2\pi} \leq 1/T$.
    Thus, as in the finite-domain case, we obtain
    \begin{align}
        {\rm BSR}_T
        &\leq
        \EE \sbr{f(\*x^*) - f([\*x^*]_T)}
        + \EE \sbr{f([\*x^*]_T) - U_T\rbr{[\*x^*]_T}}
        + \EE \sbr{U_T\rbr{[\*x^*]_T} - \mu_{T}([\*x^*]_T)} \\
        &\leq
        \underbrace{\EE \sbr{f(\*x^*) - f([\*x^*]_T)}}_{R_3}
        + \frac{1}{T}
        + \beta_T^{1/2} \underbrace{\EE \sbr{\sigma_{T}([\*x^*]_T)}}_{R_4},
    \end{align}
    where $U_T(\*x) = \mu_T(\*x) + \beta_T^{1/2}\sigma_T(\*x)$.
    Thus, we derive upper bounds of $R_3$ and $R_4$.

    For $R_3$, convexity of $\cX$ and the bound on the distance to the net imply
    \begin{align}
        R_3
        &\leq \EE\sbr{ \max_{j \in [d]} \sup_{\*x \in \cX} \left| \frac{\partial f (\*u)}{\partial u_j} \biggm|_{\*u = \*x} \right| \| \*x^* - [\*x^*]_T\|_1 } \\
        &\leq L_f \frac{1}{LT} \\
        &\leq \frac{1}{T},
    \end{align}
    from the definition of $L, L_f$ and $\cX_T$.

    For $R_4$, from Lemma~\ref{LemLipschitzPosteriorSTD}, we can apply a similar derivation as follows:
    \begin{align}
        R_4
        &\leq \EE\sbr{ |\sigma_{T}([\*x^*]_T) - \sigma_{T}(\*x^*)| + \sigma_{T}(\*x^*) } \\
        &\leq L_k \frac{1}{LT} + \EE\sbr{ \sigma_{T}(\*x^*) } \\
        &\leq \frac{1}{T} + \EE\sbr{ \sigma_{T}(\*x^*) }.
    \end{align}

    Combining the upper bounds of $R_3$ and $R_4$, we can see that
    \begin{align}
        {\rm BSR}_T
        &\leq \beta_T^{1/2} \EE\sbr{ \sigma_{T}(\*x^*) } + \frac{2 + \beta_T^{1/2} }{T}.
    \end{align}
    The remaining proof follows the finite-domain argument above, using the continuous-domain MC error bound.
    Since $c_T \leq c_1$, the additional term $\sqrt{\beta_T}/T$ is absorbed into $O\rbr{\sqrt{\beta_T}/(T c_T)}$, completing the proof.
\end{proof}

\section{Additional Experiments}
\label{sec:additional_experiments}

\subsection{Experiments for Other Benchmark and Emulator Functions}
\label{sec:exp_others}

In this section, we provide results for the other benchmark functions and emulator functions in Figures~\ref{fig:additional_exp_benchmark} and \ref{fig:additional_exp_emulator}, where the experimental settings are the same as those in the main experiments.
For benchmark function experiments, we can confirm that JES, MES, TS, and KG sometimes show inferior performance.
On the other hand, other methods often show the top-group performance.
For emulator function experiments, the BO methods do not show large differences.
For benzylation, MES accidentally shows superior performance.
For colors\_n9, all the BO methods except for KG show similar performance.
For fullerenes, one-step lookahead methods, such as OVR, ROVR, KG, and JES, show slightly worse performance than the other methods.
Although \citet{hase2021olympus} provided other emulators, we do not report the results for them since they seem to have an implementation bug.
Overall, we conclude that OVR and ROVR have comparable performance to the other BO methods.

\begin{figure*}[t]
    \centering
    \includegraphics[width=0.9\linewidth]{fig/legend.pdf}

    \includegraphics[width=0.021\linewidth]{fig/vertical_axes_label.pdf}
    \includegraphics[width=0.237\linewidth]{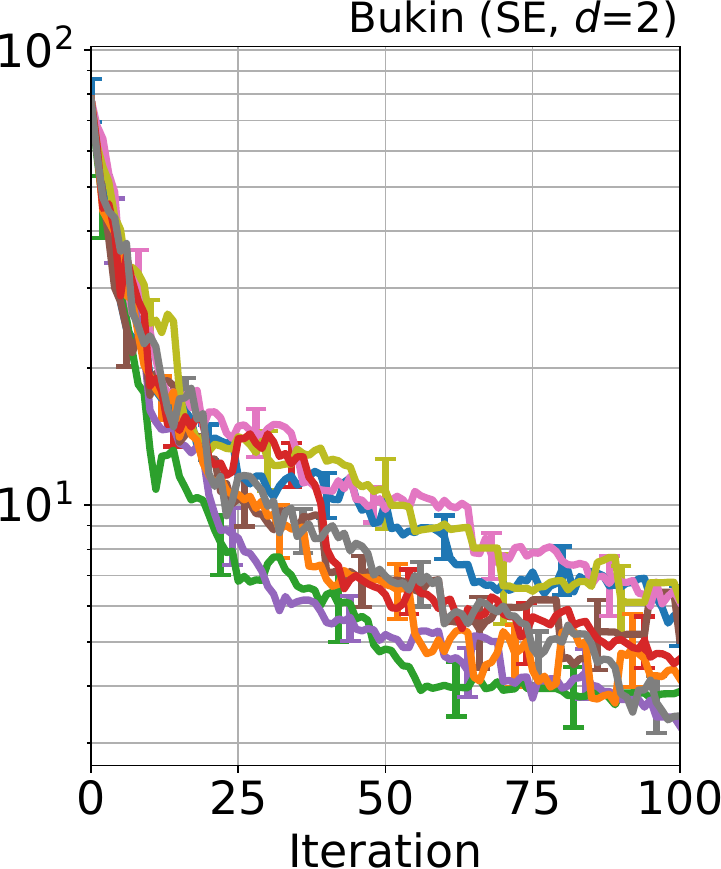}
    \includegraphics[width=0.237\linewidth]{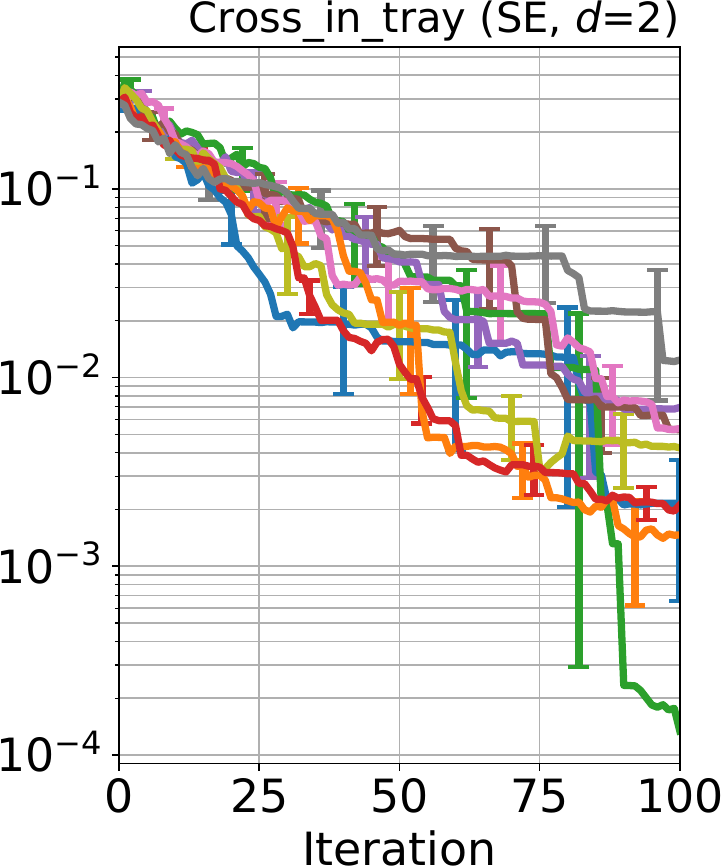}
    \includegraphics[width=0.237\linewidth]{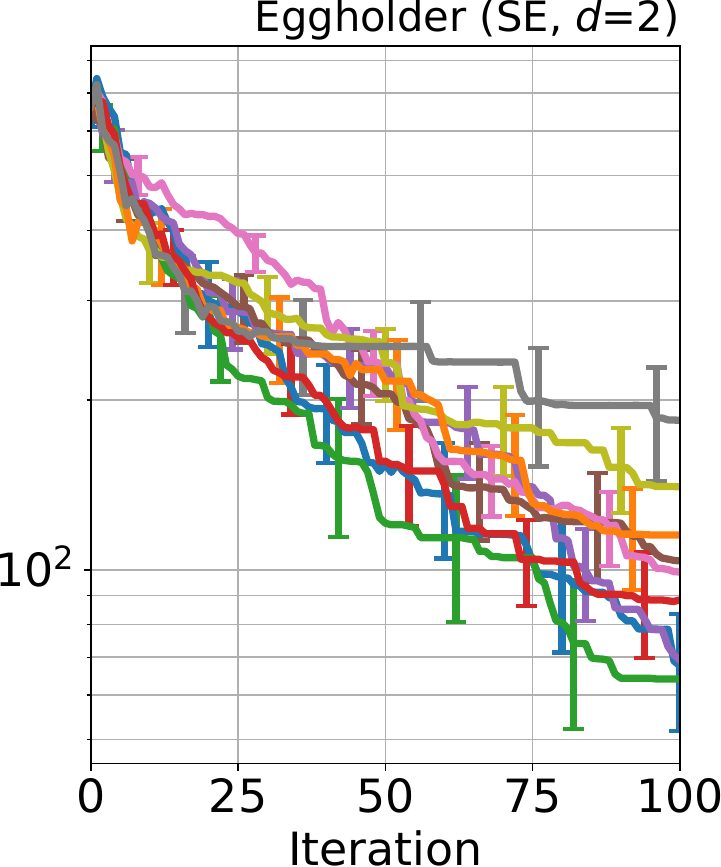}
    \includegraphics[width=0.237\linewidth]{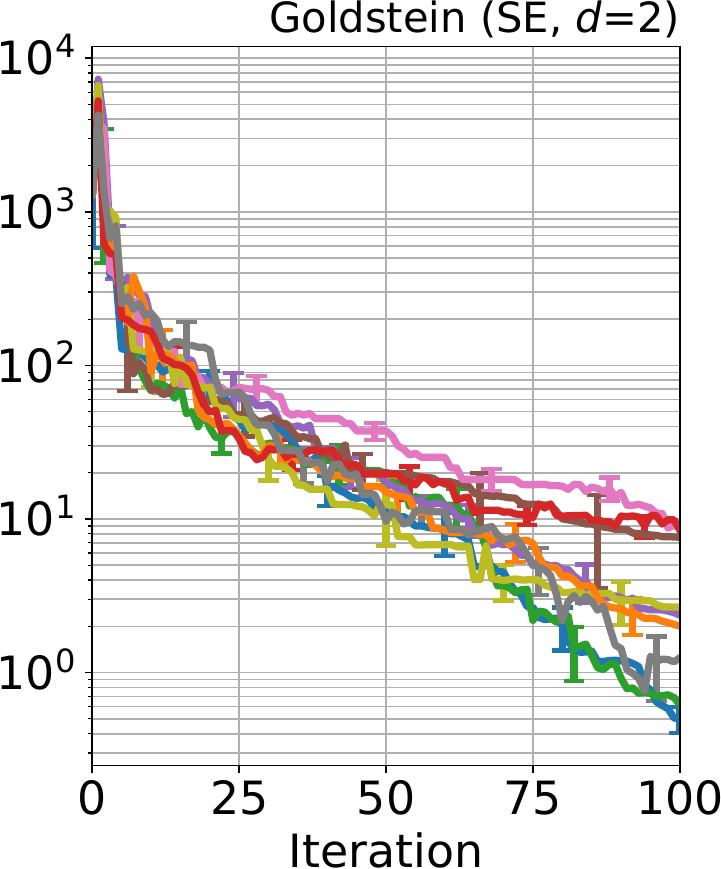}

    \includegraphics[width=0.021\linewidth]{fig/vertical_axes_label.pdf}
    \includegraphics[width=0.237\linewidth]{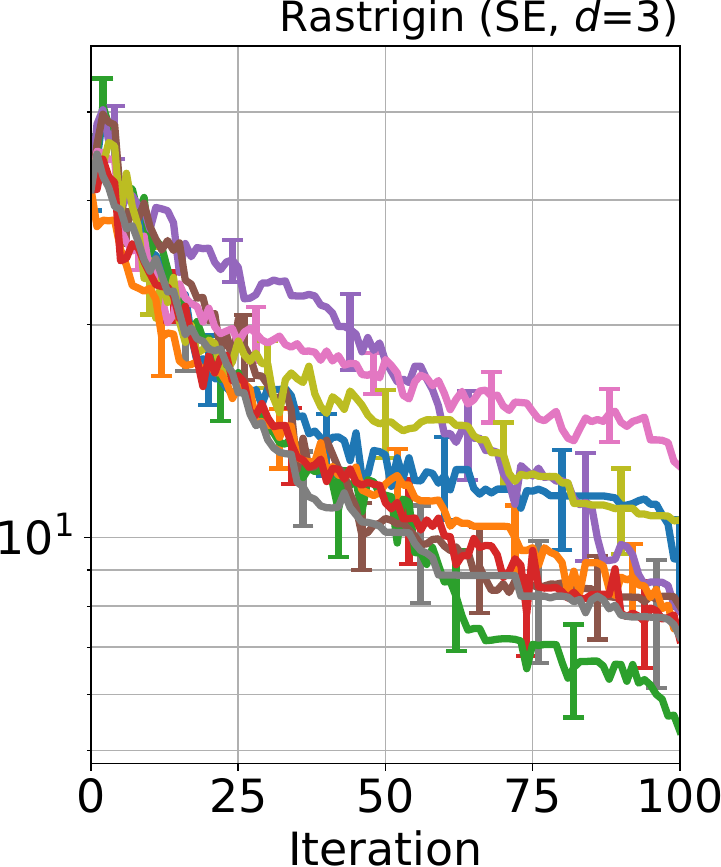}
    \includegraphics[width=0.237\linewidth]{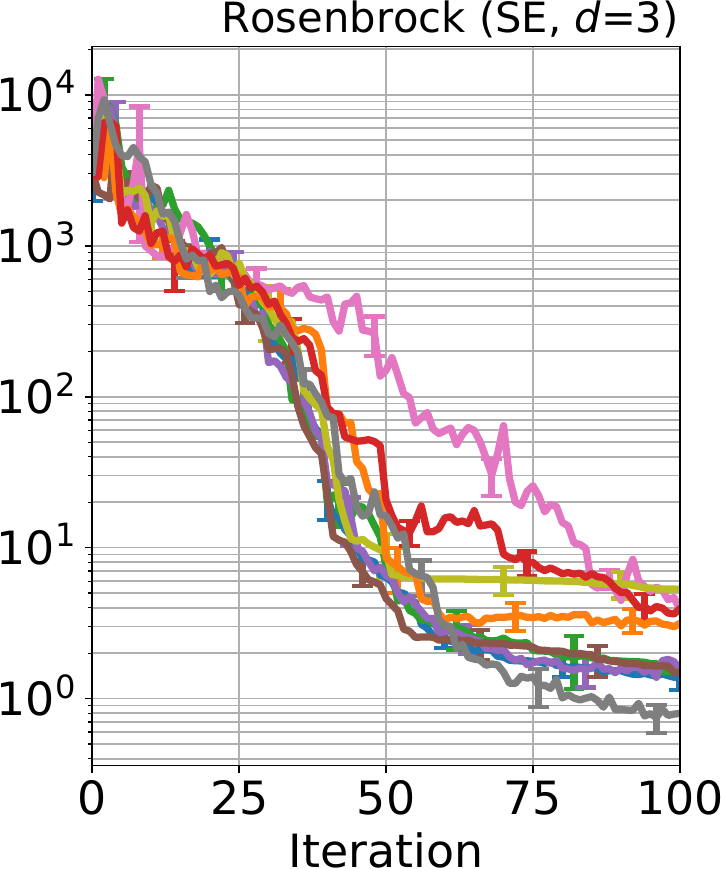}
    \includegraphics[width=0.237\linewidth]{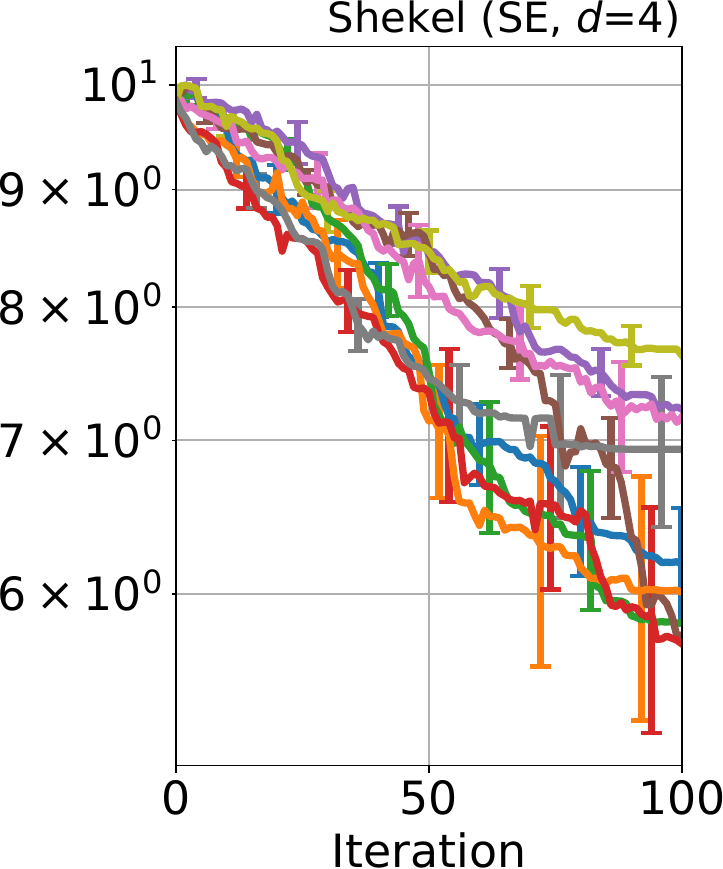}
    \includegraphics[width=0.237\linewidth]{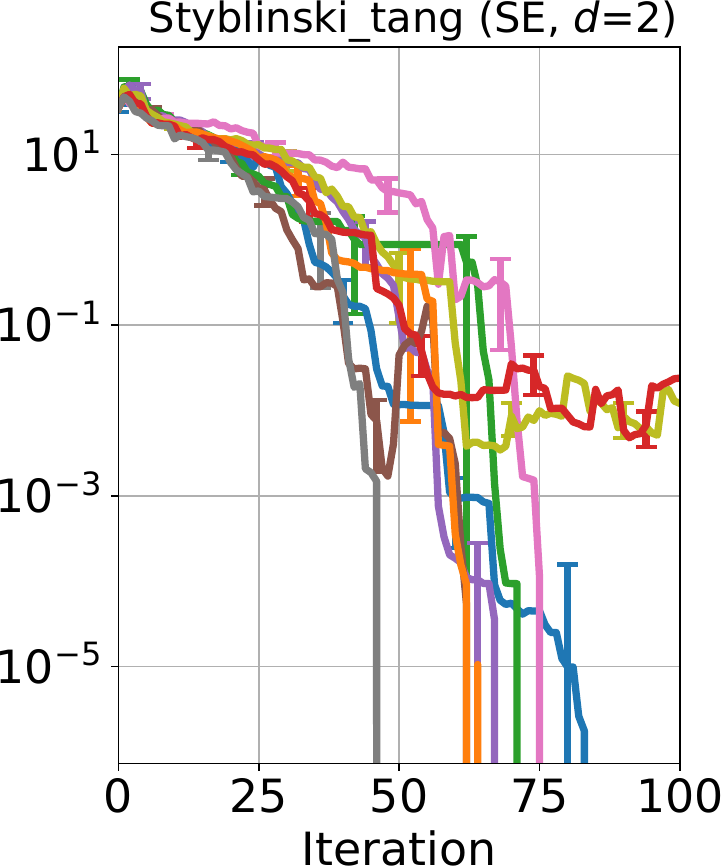}
    \caption{
    The average and standard error of experimental results for the benchmark functions.
    The numerical suffix in the legend indicates the number of MC samples.
    %
    }
    \label{fig:additional_exp_benchmark}
\end{figure*}

\begin{figure*}[t]
    \centering
    \includegraphics[width=0.9\linewidth]{fig/legend.pdf}

    \includegraphics[width=0.021\linewidth]{fig/vertical_axis_objective.pdf}
    \includegraphics[width=0.237\linewidth]{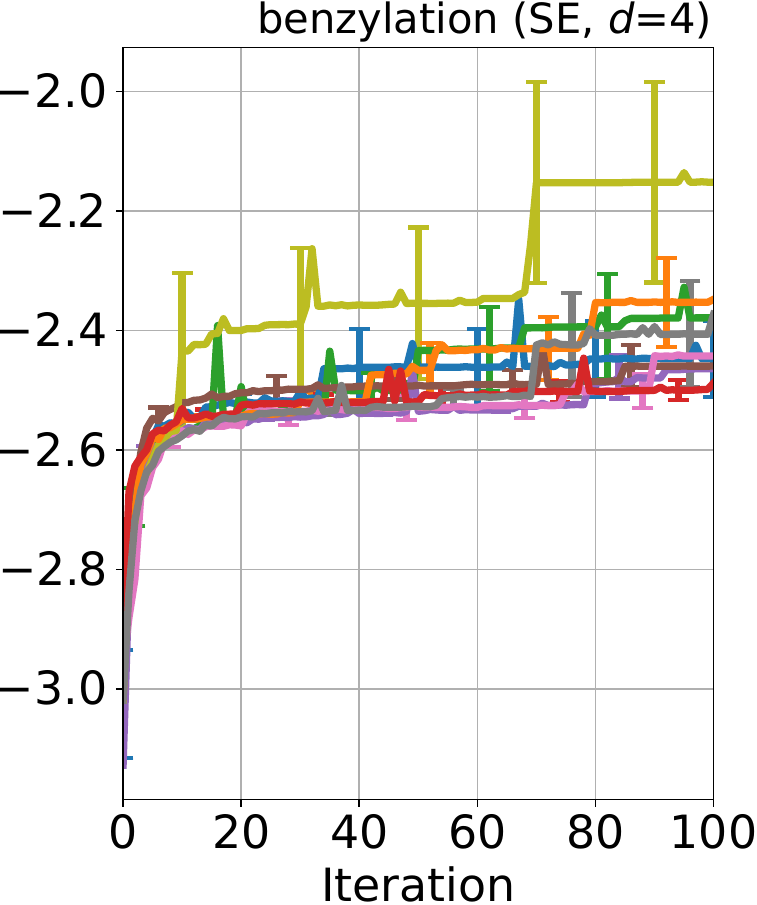}
    \includegraphics[width=0.237\linewidth]{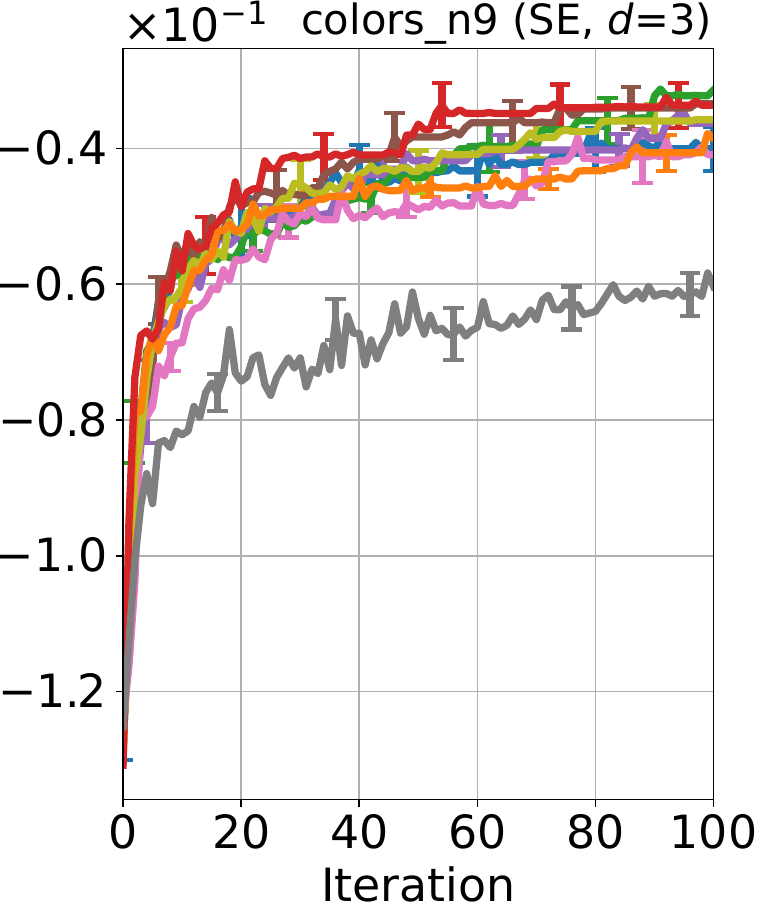}
    \includegraphics[width=0.237\linewidth]{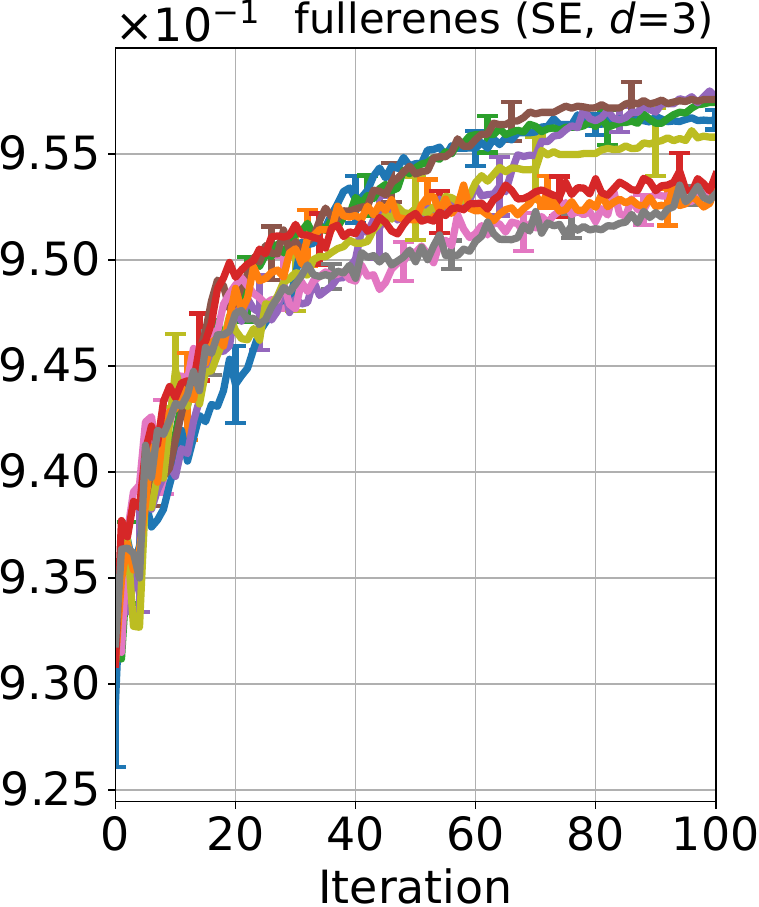}
    \caption{
    The average and standard error of experimental results for the emulator functions.
    The numerical suffix in the legend indicates the number of MC samples.
    %
    }
    \label{fig:additional_exp_emulator}
\end{figure*}

\subsection{Sensitivity to Number of MC Samples}
\label{sec:exp_sensitivity_MC}

Figure~\ref{fig:exp_synthetic_sensitivity} shows the synthetic function experiment results, where the experimental settings are the same as those of Figure~\ref{fig:exp_synthetic} except for the number of MC samples.
We observed that, regardless of the number of MC samples, MES often gets stuck in a local optimum.
We conjecture that this phenomenon occurs because we do not employ a heuristic modification of generated maximum values $\tilde{f}_t^* = \max_{\*x \in \cX} \tilde{f}_t(\*x)$, where $\tilde{f}_t \sim p(f \mid \cD_{t})$.
That is, the original implementation \citep{Wang2017-Max} uses $\max\{ \tilde{f}_t^*, y_1 + c, \dots, y_t + c \}$ with some jitter $c > 0$ instead of $\tilde{f}_t^*$ (Line 75 in \url{https://github.com/zi-w/Max-value-Entropy-Search/blob/master/utils/sampleMaximumValues.m}).
This modification often alleviates over-exploitation, since a smaller $\tilde{f}_t^*$ facilitates exploitation.
Furthermore, we do not employ such a heuristic modification for JES either.
As a result, the performance of JES often degraded as the number of MC samples increased.
When the number of MC samples increases, the probability of obtaining a smaller $\tilde{f}_t^*$ increases.
Thus, we conjecture that JES gets stuck in local optima since such effects from a smaller $\tilde{f}_t^*$ increase.
On the other hand, OVR and ROVR show robust performance with respect to the number of MC samples, with performance becoming more stable as the number of MC samples increases.

\begin{figure}[t]
    \centering
    \includegraphics[width=0.6\linewidth]{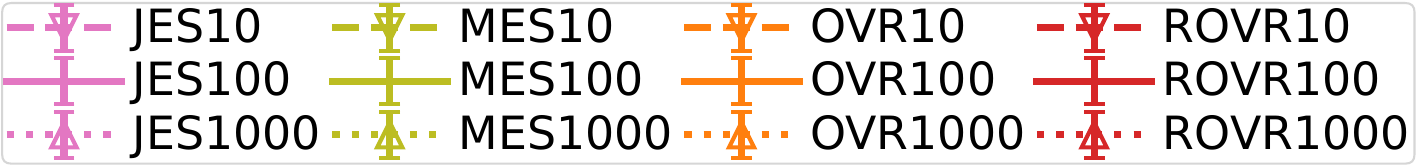}

    \includegraphics[width=0.021\linewidth]{fig/vertical_axes_label.pdf}
    \includegraphics[width=0.237\linewidth]{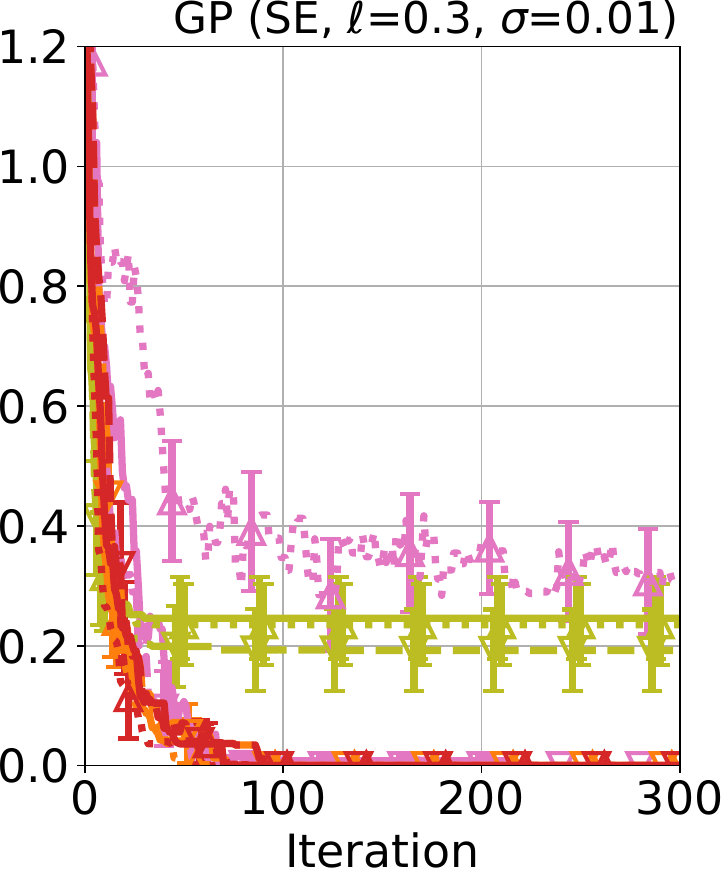}
    \includegraphics[width=0.237\linewidth]{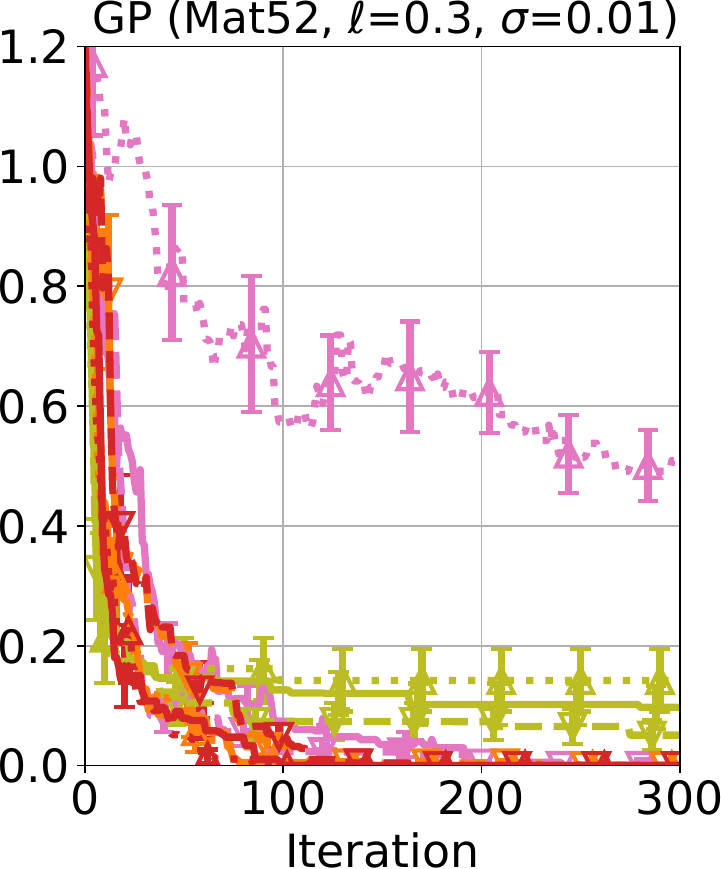}
    \includegraphics[width=0.237\linewidth]{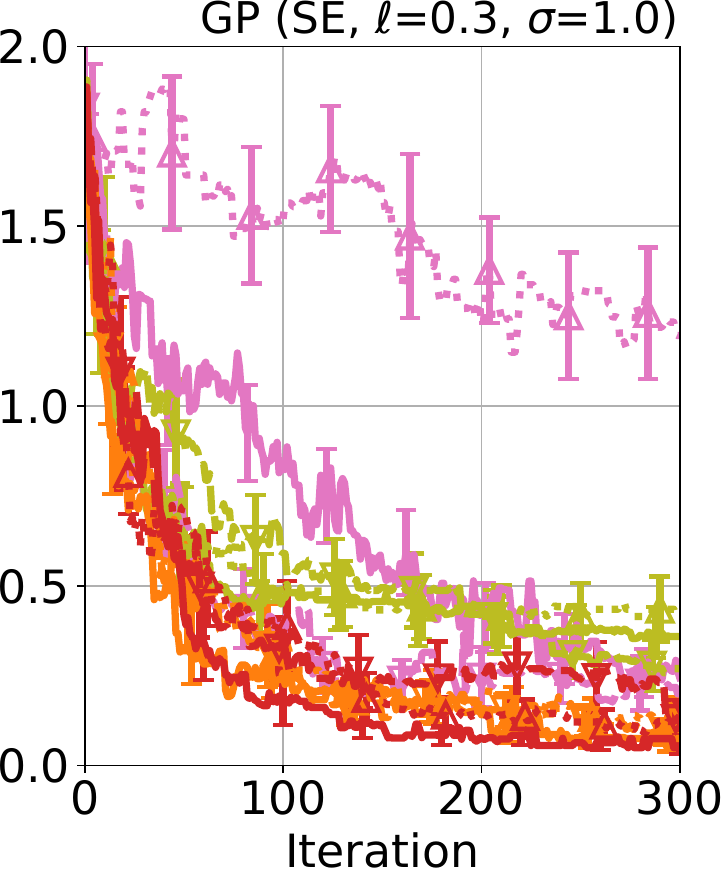}
    \includegraphics[width=0.237\linewidth]{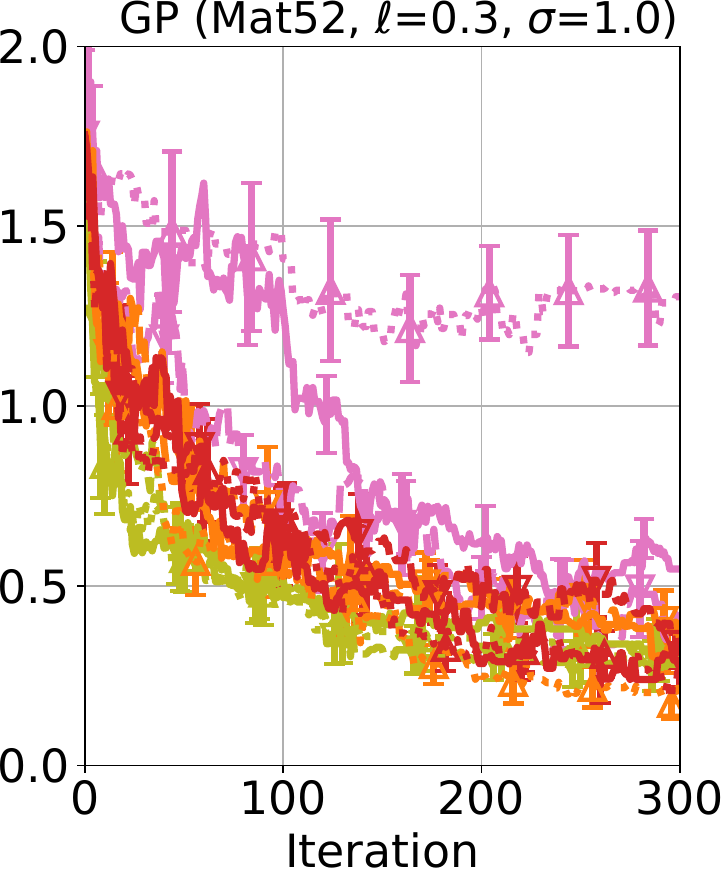}

    \caption{
    The average and standard error of simple regret in the synthetic function experiments.
    The numerical suffix in the legend indicates the number of MC samples.
    The left two plots and the right two plots show the results for $\sigma^2=10^{-4}$ and $\sigma^2=1$, respectively.
    }
    \label{fig:exp_synthetic_sensitivity}
\end{figure}
\end{document}